\documentclass{article}

\usepackage{arxiv}
\usepackage{amsfonts,booktabs,multirow,lineno,hyperref,xcolor,amsmath}
\usepackage{blkarray, bigstrut}
\usepackage{booktabs}
\newcommand\topstrut[1][1.2ex]{\setlength\bigstrutjot{#1}{\bigstrut[t]}}
\newcommand\botstrut[1][0.9ex]{\setlength\bigstrutjot{#1}{\bigstrut[b]}}

\usepackage[utf8]{inputenc} 
\usepackage[T1]{fontenc}    
\usepackage{hyperref}       
\usepackage{url}            
\usepackage{booktabs}       
\usepackage{amsfonts}       
\usepackage{nicefrac}       
\usepackage{microtype}      
\usepackage{lipsum}
\usepackage{graphicx}
\graphicspath{ {./images/} }

\usepackage{amsmath,amssymb}
\usepackage{adjustbox}
\usepackage{caption}
\usepackage{subfigure}
\usepackage{multirow}
\usepackage[english]{babel}
\usepackage{amsthm}
\newtheorem{definition}{Definition}
\newtheorem{theorem}{Theorem}
\newtheorem{proposition}{Proposition}
\newtheorem{lemma}{Lemma}
\newtheorem{corollary}{Corollary}

\title{To do or not to do: cost-sensitive causal decision-making}

\author{
Diego~Olaya \\
Solvay Business School\\
Vrije Universiteit Brussel\\
Brussels, Belgium\\
\texttt{diego.olaya@vub.be}

\And
Wouter~Verbeke\\
Faculty of Economics and Business \\
Katholieke Universiteit Leuven\\
Leuven, Belgium\\
\texttt{wouter.verbeke@kuleuven.be}

\And
Jente~Van~Belle \\
Solvay Business School\\
Vrije Universiteit Brussel\\
Brussels, Belgium\\
\texttt{jente.van.belle@vub.be}

\And
Marie-Anne~Guerry \\
Solvay Business School\\
Vrije Universiteit Brussel\\
Brussels, Belgium\\
\texttt{marie-anne.guerry@vub.be}
}

\begin{document}
\maketitle

\begin{abstract}
Causal classification models are adopted across a variety of operational business processes to predict the effect of a treatment on a categorical business outcome of interest depending on the process instance characteristics. This allows optimizing operational decision-making and selecting the optimal treatment to apply in each specific instance, with the aim of maximizing the positive outcome rate. While various powerful approaches have been presented in the literature for learning causal classification models, no formal framework has been elaborated for optimal decision-making based on the estimated individual treatment effects, given the cost of the various treatments and the benefit of the potential outcomes.

In this article, we therefore extend upon the expected value framework and formally introduce a cost-sensitive decision boundary for double binary causal classification, which is a linear function of the estimated individual treatment effect, the positive outcome probability and the cost and benefit parameters of the problem setting. The boundary allows causally classifying instances in the positive and negative treatment class to maximize the expected causal profit, which is introduced as the objective at hand in cost-sensitive causal classification. We introduce the expected causal profit ranker which ranks instances for maximizing the expected causal profit at each possible threshold for causally classifying instances and differs from the conventional ranking approach based on the individual treatment effect. The proposed ranking approach is experimentally evaluated on synthetic and marketing campaign data sets. The results indicate that the presented ranking method effectively outperforms the cost-insensitive ranking approach and allows boosting profitability.

\end{abstract}


\section{Introduction}


Causal classification models support operational decision-making by providing estimates of the effect of an action at hand on a categorical business outcome of interest at an individual instance level. For example, causal classification models can estimate the effect of a discount on the purchase propensity of an individual customer, as such facilitating optimal marketing resource allocation \cite{ben2018optimal,simester2020targeting}. More generally, in a production process setting, a causal classification model can estimate the effect of a production parameter (e.g., time, temperature, operator, channel, etc.) on the resulting process output (e.g., yield, quality, process time, etc.) depending on process instance characteristics (e.g., environmental factors, raw material characteristics, case specifications, operator experience, etc.), consequently facilitating customization or optimization of the production process and maximizing intervention effectiveness \cite{gupta2020maximizing}.

In causal machine learning, the action at hand is called a treatment and the effect on the outcome for a specific instance is called the individual treatment effect (ITE). Various causal classification methods have been proposed and a variety of case studies have been presented in fields such as marketing \cite{ascarza2018retention,debaere2019reducing}, personalized medicine \cite{berrevoets2020organite}, human resources management \cite{rombaut2020effectiveness}, and learning analytics \cite{turkyilmaz2018causal,olaya2020uplift}.
To the best of our knowledge, however, no generic decision-making method has been formalized in the literature for optimally classifying instances in treatment classes based on the ITE estimates that are produced by a causal model that account for the cost of the treatments and the benefits of the outcomes. 

In this article, we build on the foundations of cost-sensitive classification as introduced by \cite{elkan2001foundations} and the foundations of cost-sensitive causal classification as introduced by \cite{verbeke2020foundations}, in order to formulate a cost-sensitive decision boundary for causally classifying instances in treatment classes, i.e., to decide which treatment to apply to an instance, as a function of the individual treatment effect and the positive outcome probability estimates that are produced by a causal classification model, as well as the cost and benefit parameters of the problem setting. 
Moreover, we introduce a novel cost-sensitive ranking method for causal classification, which maximizes the expected profit across all possible thresholds.  We present the application of the proposed ranking method on a synthetic data set as well as on data from digital advertising and customer retention campaigns, showcasing the potential of the novel ranking to outperform the standard, cost-insensitive ranking approach that is purely based on ITE estimates and to boost profitability. In the online appendix to this article, the experimental code and data sets are provided to facilitate reproduction of the presented results.

This article is structured as follows. In the following section, we provide a comprehensive review of the foundations of cost-sensitive classification and causal classification. Subsequently, we elaborate and analyze the cost-insensitive and cost-sensitive decision boundaries for causal classification, and introduce the novel cost-sensitive ranking method. Finally, experimental results are presented and discussed, before we conclude the article and present directions for future research.

\section{Related work}\label{sec:Prior}

This section formally introduces the classification and causal classification tasks and summarizes the cost-sensitive evaluation framework to assess performance of classification and causal classification models as presented in \cite{verbeke2020foundations}, which will be elaborated upon in the following section. 

\subsection{Classification}\label{sec:classification}

Classification is a common machine learning task that concerns the assignment of instances $x$ from the instance space $X \subseteq \mathbb{R}^n$ to a class or outcome $Y$. In this article, we focus on binary classification, i.e., $Y \in \{0,1\}.$ By convention, we refer to outcome $Y=1$ as the positive class and to outcome $Y=0$ as the negative class. Instances with outcome $Y=1$ are called positive instances and with $Y=0$ negative instances. From a data set of labeled instances, $\mathcal{D}=(x_i,y_i)$ with $i=1:n$, a binary classifier can learn a binary classification model, $m: X \to [0,1]$, which maps instances $x$ to a positive outcome probability, $P(Y=1|x)$ or $P_1$ in short, and a negative outcome probability, $P(Y=0|x) = 1-P(Y=1|x)$ or $P_0$ in short.

A class estimate or prediction, $\hat{Y} \in \{0,1\}$, is obtained by setting a decision threshold $\phi$. Instances with a positive outcome probability above the threshold, i.e., $P_1 > \phi$, are classified in the positive class, whereas instances with a positive outcome probability below the threshold, i.e., $P_1 < \phi$ are classified in the negative class. The optimal cost-insensitive classification threshold, $\phi^*_{ci}$, which maximizes the expected number of correctly classified outcomes, classifies an instance in the positive class if the positive outcome probability is greater than the negative outcome probability, i.e., if $P_1 > P_0$. Hence, since $P_1 + P_0 = 1$, we have:
\begin{equation}\label{eq:costinsensitiveclassification}
\phi^*_{ci} = 0.5.     
\end{equation}

The confusion matrix, $\mathbf{CF}(\phi)$, reports for an arbitrary classification threshold $\phi$ the number of \textit{True Positives} (TP) and \textit{False Negatives} (FN), which are the number of correctly and incorrectly classified positive instances, respectively, as well as the number of \textit{True Negatives} (TN) and \textit{False Positives} (FP), which are the number of correctly and incorrectly classified negative instances.

\begin{equation} \label{eq:CF}
    \begin{blockarray}{ccccc}
        & \BAmulticolumn{2}{c}{\text{Prediction}} & & \\
        & \hat{Y}=0 & \hat{Y}=1 & & \\
        \addlinespace
        \begin{block}{c[cc]cc}
            \topstrut \multirow{2}{*}{$\mathbf{CF(\phi)}$ :=} & TN(\phi) & FP(\phi) & \textit{Y}=0 & \multirow{2}{*}{Outcome}\\
            \addlinespace
             & FN(\phi) & TP(\phi) & \textit{Y}=1 & &\\
        \end{block}
    \end{blockarray}
\end{equation}

A range of metrics for evaluating the performance of classification models, such as accuracy, F1, sensitivity and specificity, are a function of the confusion matrix and thus of the threshold $\phi$. Therefore, these metrics can be misleading, and certainly so when the class distribution is imbalanced. Alternative measures that assess the quality of the classification model without making any assumption about the operational threshold include, for example, the area under the receiver operating characteristics curve (AUC) or the area under the precision-recall curve (PRAUC). 

Commonly, classifying instances incorrectly leads to misclassification costs, whereas correct classifications yield a benefit. All the above performance measures implicitly assume that false negatives and false positives imply equal misclassification costs and that true negatives and true positives imply equal benefits. In many problem settings, however, correctly and incorrectly classified instances within the negative and positive class involve variable costs and benefits, respectively. Cost-sensitive learning methods take into account these benefits and costs in learning a classification model or in setting a classification threshold and cost-sensitive performance measures in evaluating model performance. 

In this article, we focus on cost and benefit parameters that are constant across all instances of a class, i.e., class-dependent cost and benefit parameters. An extension to instance-dependent cost and benefit parameters is listed as a topic of interest for further research. The cost or benefit of classifying an instance from class $i$ in class $j$ is denoted by $cb_{ij}$\footnote{In literature, typically the benefit of a true negative and true positive are denoted with $b_0$ and $b_1$, respectively, and the cost of a false negative and a false positive with $c_0$ and $c_1$, respectively. For achieving a consistent notation across conventional and causal classification, we introduce the general notation $cb_{ij}$.}. A positive value of $cb_{ij}$ represents a benefit and a negative value a cost. The cost-benefit parameters are summarized in the cost-benefit matrix, \textbf{CB}:   

\begin{equation} \label{eq:CB_V1}
    \begin{blockarray}{ccccc}
        & \BAmulticolumn{2}{c}{\text{Prediction}} & & \\
        & \hat{Y}=0 & \hat{Y}=1 & & \\
        \addlinespace
        \begin{block}{c[cc]cc}
            \topstrut \multirow{2}{*}{$\mathbf{CB}$ :=} & cb_{00} & cb_{01} & Y=0 & \multirow{2}{*}{Outcome}\\
           \addlinespace
            & cb_{10} & cb_{11} & Y=1 & &\\
        \end{block}
    \end{blockarray}
\end{equation}

\cite{elkan2001foundations} establishes the optimal cost-sensitive threshold that maximizes the expected profit, $\phi^*_{cs}$, for classifying instances as a function of the cost-benefit parameters that are specified in the $\mathbf{CB}$ matrix. An instance $x_i$, with $P_{y,i}=P(Y=y|x_i)$, is to be classified in the positive class if the expected profit of classifying the instance in the positive class is greater than the expected profit of classifying the instance in the negative class, i.e., if:

\begin{equation} \label{eq:optimal_threshold_classification}
    P_{1,i} cb_{11} + P_{0,i} cb_{01} > P_{0,i} cb_{00} + P_{1,i} cb_{10},
\end{equation}

\noindent or, equivalently, when the positive outcome probability, $P_{1,i}$ is greater than the cost-sensitive classification threshold, $\phi_{cs}^*$, which can be obtained from Equation \ref{eq:optimal_threshold_classification} and is defined as: 

\begin{equation}
    \phi_{cs}^* = \frac{cb_{00} - cb_{01}}{cb_{11}+cb_{00}-cb_{01}-cb_{10}}.
    \label{eq:cost_theoretical_threshold}
\end{equation}

Note that using the optimal cost-sensitive classification threshold requires the classification model to produce calibrated probability estimates rather than scores, which merely allow ranking of the instances \cite{zadrozny2001obtaining,sheng2006thresholding}. 

\subsection{Causal classification}

Causal classification involves the assignment of instances $x$ from the instance space $X \subseteq \mathbb{R}^n$ to a class or outcome $Y$, as a function of a treatment $W$. That is, causal classification involves the estimation of the causal effect of a treatment on an outcome of interest at the level of an individual instance, i.e., the individual treatment effect.

Causal classification has been addressed in the literature as uplift modeling \cite{ascarza2018retention,devriendt2019you}, heterogeneous or individual treatment effect estimation \cite{wager2018estimation}, individualized treatment rule learning \cite{zhao2012estimating} and conditional average treatment effect estimation \cite{shalit2017estimating}. Causal classification is an instance of counterfactual estimation and causal learning \cite{pearl2009causality}. A subtle difference across various approaches is in the objective, which can either be to predict or to explain \cite{shmueli2010explain}, i.e., to optimize the treatments that are applied at the individual instance level, which is the objective in this article, or to establish and measure the strength of causal relations among a set of variables, which is the objective in causal inference and structural equation modeling. 

Causal classification methods for learning causal classification models include metalearners and modified classification methods. Metalearners are essentially schemes that allow the adoption of standard machine learning algorithms for estimating the ITE (see \cite{kunzel2019metalearners} for an overview). By contrast, various classification methods have been modified to directly learn models for estimating the ITE, such as neural networks \cite{shalit2017estimating}, support vector machines \cite{kuusisto2014support}, causal Bayesian additive regression trees \cite{hill2011bayesian}, causal boosting \cite{powers2018some}, causal forests \cite{athey2019generalized}. Moreover, multi-armed bandits have been proposed to estimate the optimal treatment by considering individuals' attributes or estimating responses based on past actions \cite{berrevoets2019optimising}. The main task of traditional multi-armed bandits, however, consists of maximizing the reward from an action (i.e., a treatment), rather than the individual-level effect or net increase in reward \cite{zhao2019uplift}. 

In this study, we employ the Neyman-Rubin framework to estimate individual treatment effects in terms of potential outcomes \cite{rubin1974estimating,splawa1990application}, and we focus on double binary causal classification, i.e., $Y \in \{0,1\}$ and $W \in \{0,1\}$. In line with the convention for the outcome variable, we refer to treatment $W=1$ as the positive treatment and to treatment $W=0$ as the negative treatment. The outcome of instance $x_i$ for treatment $w$ is denoted by $Y_{w,i}$. The individual treatment effect is obtained by contrasting the outcome for the positive and the negative treatment, i.e., $ITE = Y_{1,i} - Y_{0,i}$. 

In real-world conditions, the ITE cannot be observed, since only a single treatment at a time can be applied and hence a single outcome, $Y_{w,i}$, is observed. This is known as the fundamental problem of causal inference \cite{holland1986statistics}. The ITE is therefore estimated under the assumptions of ignorability, common support and stable unit treatment value (SUTVA) \cite{rubin1978bayesian,rosenbaum1983central}. In principle, a causal classification method requires two independent and identically distributed samples that are obtained by means of a randomized controlled trial\footnote{Various methods for learning causal models from observational data that suffers from selection bias, however, have been proposed in literature, for instance, by sampling or learning a balanced representation using propensity score matching \cite{caliendo2008some} or domain adversarial training \cite{ganin2016domain}} and that are representative of the population of interest for learning a causal classification model. A first sample is typically called the \textit{treatment group}, $T$, and concerns a set of instances that have been treated with the positive treatment, $W=1$. A second sample concerns instances that have been treated with the negative treatment, $W=0$, and is typically called the \textit{control group}, $C$. Often, the negative treatment represents a \textit{no treatment} scenario (e.g., no discount is offered, a placebo is given, etc.) and the control group serves as a baseline for evaluating the effect of a potential treatment, i.e., the positive treatment. The number of instances with outcome $y$ in the treatment and control group are denoted by $T_{y}$ and $C_{y}$, respectively.

A causal classification model is a function $\dot{m}: X \to t \in [-1,1]$, with $t$ the estimated ITE which is calculated as:

\begin{equation} \label{eq:estimated_ITE}
    t(x) := P_{11} - P_{10}.
\end{equation}

\noindent and with $P_{yw} = P(Y=y|x,W=w)$. Note that by convention, the estimated ITE represents the change in the positive outcome probability that is caused by applying the positive treatment, reflecting the objective of maximizing the positive outcome rate among the population by means of applying the positive treatment on a subset of the population \cite{fernandez2019causal}. In business decision-making applications, the objective in cost-insensitive causal classification is to identify instances for which application of the positive treatment results in the outcome of the instance changing from the negative to the positive class, i.e., instances with $ITE = 1$. Whereas in uplift modeling \cite{lo2002true}, such instances are called \emph{persuadables}, we introduce and adopt the term \textit{positive converted instance}. For instance, in marketing, the objective is to identify customers (i.e., the instances) who would not purchase (i.e., the negative outcome) without receiving a discount (i.e., the treatment), and who will purchase (i.e., the positive outcome) because of receiving a discount. Note that additionally, it may be equally important to identify instances that are adversely affected by the positive treatment, i.e., with $ITE = -1$. In marketing, for instance, some customers are known to adversely respond to retention campaigns, which are designed to retain customers but may also cause customer churn \cite{ascarza2018retention,devriendt2019you}. These customers are sometimes referred to as \emph{do-not-disturbs}, whereas we propose and adopt the general term \emph{negative converted instance}. Instances for which the outcome does not change because of the treatment will be called positive or negative nonconverted instances, depending on whether the outcome is positive or negative, respectively. In uplift modeling, these instances are sometimes referred to as \textit{sure things} and \textit{lost causes}, respectively.

\subsection{Causal profit}

To causally classify instances in the positive or negative treatment class based on the estimated ITE, i.e., to decide which treatment to apply to each instance, a threshold $\tau$ is required. Instances with $t \geq \tau$ are classified in the positive treatment class and with $t < \tau$ in the negative treatment class. The causal confusion matrix, $\mathbf{\dot{CF}}(\tau)$, is the causal equivalent of the $\mathbf{CF}$ matrix
and a function of the threshold $\tau$ that is adopted \cite{verbeke2020foundations}. It reports the proportion of instances in the treatment group with negative and positive outcomes and estimated ITE above the threshold, and the proportion of instances in the control group with negative and positive outcomes and estimated ITE below the threshold. The proportion of instances with $t < \tau$ and outcome $y$ are denoted by $C_y(\tau)$ for the control group and $T_y(\tau)$ for the treatment group.

\begin{equation} \label{eq:CE}
    \begin{blockarray}{cccccc}
        & \BAmulticolumn{2}{c}{\text{Treatment}} & & \\
        \addlinespace
        & W=0 & W=1 & & \\
        \addlinespace
        \begin{block}{c[cc]ccc}
            \topstrut \multirow{2}{*}{$\mathbf{\dot{CF}(\tau)}:=$} & \frac{C_0(\tau)}{C_0+C_1} & \frac{T_0-T_0(\tau)}{T_0+T_1} & \textit{Y}=0 & \multirow{2}{*}{Outcome}\\
            \addlinespace
            & \botstrut \frac{C_1(\tau)}{C_0+C_1}  & \frac{T_1-T_1(\tau)}{T_0+T_1}  & \textit{Y}=1\\
        \end{block}
    \end{blockarray}
\end{equation}
For assessing the effect of applying the positive treatment to the set of instances in the positive treatment class, relative to a baseline scenario of applying the negative treatment to all instances, the causal effect matrix, $\mathbf{\dot{E}(\tau)}$, is defined as the difference between the causal confusion matrix at threshold $\tau$ and the baseline confusion matrix for classifying all instances in the negative treatment class, $\mathbf{\dot{CF}(\infty)}$:
\begin{equation}
    \mathbf{\dot{E}} := \mathbf{\dot{CF}(\tau)} - \mathbf{\dot{CF}(\infty)}.
\end{equation}
Note that the baseline causal confusion matrix is obtained for threshold $\tau=\infty$.

Because we do not observe the ITE, we cannot evaluate the performance of causal classification models by comparing predictions against the ground truth, as we do in evaluating conventional classification models. Various approaches for evaluating causal models have been proposed in the literature, as surveyed in \cite{devriendt2020learning}. \cite{verbeke2020foundations} introduces a formalized framework for assessing the performance of causal classification models both in a cost-sensitive and cost-insensitive manner, which instantiates to a range of existing and novel performance measures for evaluating causal classification models. For arriving at a cost-sensitive evaluation, the framework defines the outcome-benefit matrix, $\mathbf{OB}$, and the treatment-cost matrix, $\mathbf{TC}$. The $\mathbf{OB}$ matrix specifies the benefit of an outcome given the applied treatment, with $b_{ij}$ the benefit of outcome $i$ when treatment $j$ is applied.

\begin{equation} \label{eq:OB}
    \begin{blockarray}{cccccc}
        & \BAmulticolumn{2}{c}{\text{Treatment}} & & \\
        \addlinespace
        & W=0 & W=1 & & \\
        \addlinespace
        \begin{block}{c[cc]ccc}
             \topstrut \multirow{2}{*}{$\mathbf{OB}:=$}  & b_{00} & b_{01} & Y=0 & \multirow{2}{*}{Outcome}\\
            \addlinespace
            & \botstrut b_{10} & b_{11} & Y=1\\
      \end{block}
    \end{blockarray}
\end{equation}

\noindent Similarly, the $\mathbf{TC}$ matrix specifies the cost, $c_{ij}$, of outcome $i$ given application of treatment $j$.

\begin{equation} \label{eq:AC}
    \begin{blockarray}{cccccc}
        & \BAmulticolumn{2}{c}{\text{Treatment}} & & \\
        \addlinespace
        & W=0 & W=1 & & \\
        \addlinespace
        \begin{block}{c[cc]ccc}
             \topstrut \multirow{2}{*}{$\mathbf{TC} :=$} & c_{00} & c_{01} & Y=0 & \multirow{2}{*}{Outcome}\\
            \addlinespace
            & \botstrut c_{10} & c_{11} & Y=1\\
      \end{block}
    \end{blockarray}
\end{equation}
\noindent The causal cost-benefit matrix, $\mathbf{\dot{CB}}$, is obtained by subtracting the treatment-cost matrix from the outcome-benefit matrix. It is the causal counterpart of the cost-benefit matrix in conventional classification, $\mathbf{CB}$:

\begin{equation}
\mathbf{\dot{CB}} = \mathbf{OB} - \mathbf{TC}.
\end{equation}

Note that, since by definition both benefits in $\mathbf{OB}$ and costs in $\mathbf{TC}$ are represented by positive values, an element $\dot{cb}_{ij}$ in the $\mathbf{\dot{CB}}$ matrix with a positive value represents a benefit and with a negative value a cost. The elements of the causal cost-benefit matrix directly relate to the elements of the causal confusion matrix, $\mathbf{\dot{CF}}$. Rather than using the $\mathbf{\dot{CB}}$ matrix, the decomposition in terms of the outcome-benefit and treatment-cost matrices offers practical advantages toward the analysis of the causal classification boundary in terms of its elementary constituents, i.e., the cost and benefit parameters, $c_{ij}$ and $b_{ij}$, as will be discussed in the following section.

The causal profit metric, $\dot{P}(\tau)$, evaluates the performance of a causal model in terms of the average profit per instance that is obtained when causally classifying instances using threshold $\tau$. $\dot{P}(\tau)$ is a function of the cost and benefit parameters, $c_{ij}$ and $b_{ij}$, and the elements of the causal effect matrix, $\dot{e}_{ij}$:  

\begin{align}\label{eq:profit}
    \dot{P}(\tau) &= \sum_{i,j}(\dot{\mathbf{E}}_{ij} \circ \dot{\mathbf{CB}}_{ij}), \nonumber \\
    &= \dot{e}_{00}\dot{cb}_{00} + \dot{e}_{01}\dot{cb}_{01} + \dot{e}_{10}\dot{cb}_{10} +  \dot{e}_{11}\dot{cb}_{11}, \nonumber \\
    &= \dot{e}_{00}(b_{00}-c_{00}) + \dot{e}_{01}(b_{01}-c_{01}) + \dot{e}_{10}(b_{10}-c_{10}) + \dot{e}_{11}(b_{11}-c_{11}).
\end{align}

Since $\mathbf{\dot{CF}}$ and thus $\mathbf{\dot{E}}$ are functions of the threshold $\tau$, the causal profit is a function of the threshold as well. Therefore, we will evaluate the average causal profit across all possible thresholds as well as the maximum causal profit at the optimal threshold in the experiments in Section \ref{sec:Experiments}. Both measures are defined as a function of the proportion of instances in the positive treatment class, $\eta$, which is equivalent with the use of a threshold. The average causal profit, $\dot{AP}$, is defined as:
\begin{equation}
    \dot{AP} = \int_{\eta} \dot{P}(\eta)d\eta.
\end{equation}
Note that visualizing the causal profit, $\dot{P}$, as a function of the proportion, $\eta$, yields the profit curve, with the average causal profit equal to the area underneath the curve. 

The maximum causal profit, $\dot{MP}$, is defined as:
\begin{align}\label{eq:max_profit_causal}
    \dot{MP} &:= \max_{\eta}\big(\dot{P}(\tau;\mathbf{OB},\mathbf{TC})\big), \nonumber \\
    &:= \dot{P}(\eta^*;\mathbf{OB},\mathbf{TC})
\end{align}
with $\eta^*$ the profit-maximizing threshold.

\section{Causal classification decision boundary}
\label{sec:proposal}

In this section, we introduce the cost-insensitive and cost-sensitive decision boundary for double binary causal classification, which are the causal counterparts of the cost-insensitive and cost-sensitive decision thresholds for conventional classification, respectively, as discussed in Section \ref{sec:Prior}. 

Note that, with regard to the adopted terminology, a decision threshold such as the classification threshold discussed in Section \ref{sec:Prior}, applies to a single variable that is considered in making a binary decision, whereas a decision boundary takes multiple variables into account and can be regarded as a multidimensional decision threshold.

\subsection{Cost-insensitive decision boundary}

Given the conditional class probability estimates of a causal classification model, $P_{11}$ and $P_{10}$, a decision boundary is needed to causally classify instances, i.e., to prescribe treatment $W=1$ or $W=0$. For establishing the optimal decision boundary, we first need to specify the objective that is to be optimized. 

In the cost-insensitive setting, we ignore the cost of treatments as specified in the $\mathbf{TC}$ matrix, as well as the benefits of outcomes as specified in the $\mathbf{OB}$ matrix. Implicitly, however, a positive outcome is acknowledged to have a larger benefit than a negative outcome and therefore is preferred. Hence, the straightforward objective is to maximize the positive outcome rate among the population, regardless of the total cost of treatment.

Given the estimates that are produced by a causal classification model, this objective will be achieved by applying the treatment to all instances with an estimated individual treatment effect that is positive. When the estimated individual treatment effect for an instance is negative, meaning that the treatment is expected to reduce the probability of obtaining a positive outcome for the instance, or when it is zero, then no treatment is to be applied. 

The optimal causal classification threshold, i.e., the cost-insensitive causal decision boundary, is found by conditioning the application of the treatment on the positive outcome probability given the positive treatment to be larger than the positive outcome probability given the negative treatment. Formally, instance $x_i$ is to be causally classified in class $W=1$, if:

\begin{equation*}
    P_{11,i} > P_{10,i},
\end{equation*}
i.e., if:
\begin{equation}
    t(x_i) > 0
\end{equation}

\noindent Hence, we obtain the cost-insensitive decision boundary for causal classification:

\begin{align} \label{eq:insensitive_optimal}
    \centering
    \tau^*_{ci} = 0
\end{align}

Figure \ref{fig:insensitive_frontier} visualizes the cost-insensitive decision boundary for causal classification as the horizontal axis $t = 0$ in the coordinate system $(P_{11},t)$. The positive treatment is prescribed for instances with a strictly positive value for $t$, which is visualized by the dotted area. The striped triangle B concerns instances with a negative estimated treatment effect, $t < 0$, and hence, the treatment is not to be applied. Note that the shaded triangles A and C correspond to infeasible combinations of $P_{11}$ and $t$; given that $t = P_{11} - P_{10}$ and $P_{10} \in [0,1]$, it follows that $t \leq P_{11}$ and $t \geq P_{11}-1$. 

\begin{figure}[ht]
 \centering
 \includegraphics[width=0.3\linewidth]{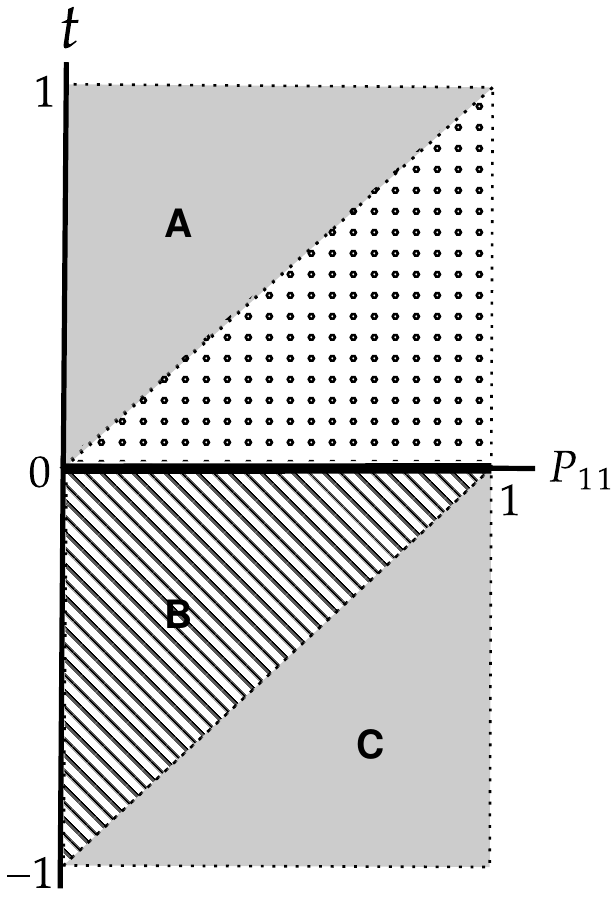}
 \caption{The cost-insensitive decision boundary for causal classification models corresponds to the horizontal axis $t = 0$ in the coordinate system $(P_{11},t)$. 
 }
 \label{fig:insensitive_frontier}
\end{figure}

The larger the estimated ITE is, the stronger the confidence of the causal classification model in predicting a positive treatment effect to occur, i.e., for the outcome to change from the negative to the positive class because of applying the positive treatment.
Hence, in case of a limited number of treatments that can be applied, e.g., because of a limited budget, it would make sense to rank and select instances from large to small estimated ITE. This ranking approach underlies visual evaluation measures for assessing the performance of causal classification measures, such as the Qini and liftup curve, as well as the associated performance metrics, i.e., the Qini coefficient and the p-percentile liftup measure (e.g., $5\%$ liftup). These curves visualize the cumulative effect of applying the treatment to an increasing proportion of the population, ranked from large to small estimated ITE, and thus support users in assessing the performance of a causal classification model at a threshold of interest for the particular application at hand, e.g., depending on the available budget. In Figure \ref{fig:insensitive_frontier}, a constraint on the number of instances that can be treated would yield a horizontal decision boundary at some value $t = t_c \geq 0$, resulting in a smaller dotted area. 

In case of a flexible budget or when the objective is to maximize net returns, users will intuitively balance the total cost of treatment, which is proportional to the selected proportion of instances, and the benefit of an increase in the number of positive outcomes, as visualized by these curves across all possible thresholds, in assessing performance and establishing the optimal threshold. Obviously, a more rigorous approach for establishing the optimal threshold and for maximizing the net returns is desirable, such as elaborated for conventional classification by \cite{elkan2001foundations} and discussed in Section \ref{sec:Prior}. Based on the foundational framework for cost-sensitive evaluation of causal classification model performance, as proposed in \cite{verbeke2020foundations}, we derive the optimal cost-sensitive decision boundary for causal classification in the following section.

\subsection{Cost-sensitive decision boundary}

Similar to the objective in finding the cost-sensitive decision threshold in conventional classification, as discussed in Section \ref{sec:Prior}, the objective in finding the cost-sensitive decision boundary in causal classification is to maximize the increase in expected profit because of applying the treatment to the subset of instances that are causally classified in the positive treatment class, which is denominated the expected causal profit and formally defined below in terms of the expected profit, $E(P|x,w)$, which is obtained for causally classifying instance $x$ in treatment class $W=w$.

\begin{definition}\label{def:expectedprofit}
The expected profit, $E(P|x,w)$, of causally classifying instance $x$ in treatment class $W=w$ is defined as:
\begin{equation}\label{eq:expectedprofit}
    E[P|x,w] := P_{1w} (b_{1w}-c_{1w}) + P_{0w} (b_{0w}-c_{0w}),
\end{equation}
with $P_{1w}+P_{0w}=1$ or $P_{0w} = 1-P_{1w}$
\end{definition} 

The expected profit, $E[P|x,w]$, is a function of the cost of the treatments, $c_{ij}$, and the benefit of the outcomes, $b_{ij}$, as specified in the treatment-cost and outcome-benefit matrices, respectively, as well as the treatment-conditional positive outcome probabilities, $P_{1w}$, estimated by the causal classification model at hand. As indicated in Section \ref{sec:Prior}, the calculation of expected profits requires calibrated probability estimates rather than scores.

\begin{definition}\label{def:expectedcausalprofit}
The expected causal profit, $E[\dot{P}|x]$, is the difference between the expected profit of classifying an instance in the positive treatment class minus the expected profit of classifying the instance in the negative treatment class, i.e.:
\begin{equation}\label{eq:expectedcausalprofit} 
    E[\dot{P}|x] := E[P|x,1] - E[P|x,0].
\end{equation}
\end{definition}

When establishing the objective in cost-sensitive causal classification as the maximization of the expected causal profit, then for each instance, the treatment $W=w$ is to be applied that maximizes the expected causal profit\footnote{We retrieved a single, earlier instantiation of Equation \ref{eq:expectedcausalprofit} in the literature, applied to in the case of customer retention, in \cite{provost2013data}, Chapter 11, page 283.}. Hence, in double binary causal classification, instances with a positive value for the expected causal profit are to be classified in the positive treatment class.

In other words, an instance is to be classified in the positive treatment class if and only if the expected profit of applying the positive treatment, $E[P|x,1]$, is greater than the expected profit of applying the negative treatment, $E[P|x,0]$, i.e., if:

\begin{equation}\label{eq:base_optimal}
    P_{11} (b_{11}-c_{11}) + (1-P_{11}) (b_{01}-c_{01}) > P_{10} (b_{10}-c_{10}) + (1-P_{10}) (b_{00}-c_{00}).
\end{equation}

The cost-sensitive causal classification decision boundary is therefore defined as the set of instances with a value for the expected causal profit equal to zero, i.e., with the expected profit for either causal classification to be equal:

\begin{equation}\label{eq:indifferent_optimal}
    P_{11} (b_{11}-c_{11}) + (1-P_{11}) (b_{01}-c_{01}) = P_{10} (b_{10}-c_{10}) + (1-P_{10}) (b_{00}-c_{00}),
\end{equation}
or, with the expected causal profit equal to zero:
\begin{equation}
E[\dot{P}|x] = 0.    
\end{equation}

By rearranging Equation \ref{eq:indifferent_optimal} and introducing the \textit{cost-benefit ratios} gamma and delta, we arrive at the following equation of the cost-sensitive causal classification decision boundary:

\begin{equation}\label{eq:ITE_threshold}
    \tau^*_{cs} = \gamma + \delta P_{11}
\end{equation}
With:
\begin{equation}\label{eq:gamma}
    \gamma = \frac{b_{00} - c_{00} - b_{01} + c_{01}}{b_{10} - c_{10} - b_{00} + c_{00}}
\end{equation}
\begin{equation}\label{eq:delta}
    \delta = \frac{b_{10} - c_{10} + b_{01} - c_{01} - b_{11} + c_{11} - b_{00} + c_{00}}{b_{10} - c_{10} - b_{00} + c_{00}}
\end{equation}    

Hence, treatment $W=1$ is to be prescribed if and only if the estimated ITE of an instance is greater than the minimum required $\tau^*_{cs}$, which is a linear function of the conditional probability $P_{11}$ and the cost-benefit ratios, $\gamma$ and $\delta$. The cost-sensitive causal classification decision boundary, defined by Equation \ref{eq:ITE_threshold}, represents the minimum individual treatment effect that is required for the application of the positive treatment to be optimal in terms of expected profit. It is a linear function of the probability to obtain a positive outcome when the positive treatment is applied and the cost-benefit structure of the problem in terms of the parameters $b_{ij}$ and $c_{ij}$. 

\begin{figure}[ht]
 \centering
 \subfigure[$\delta \in \mathbb{R}^{+}$]{\includegraphics[width=0.3\textwidth]{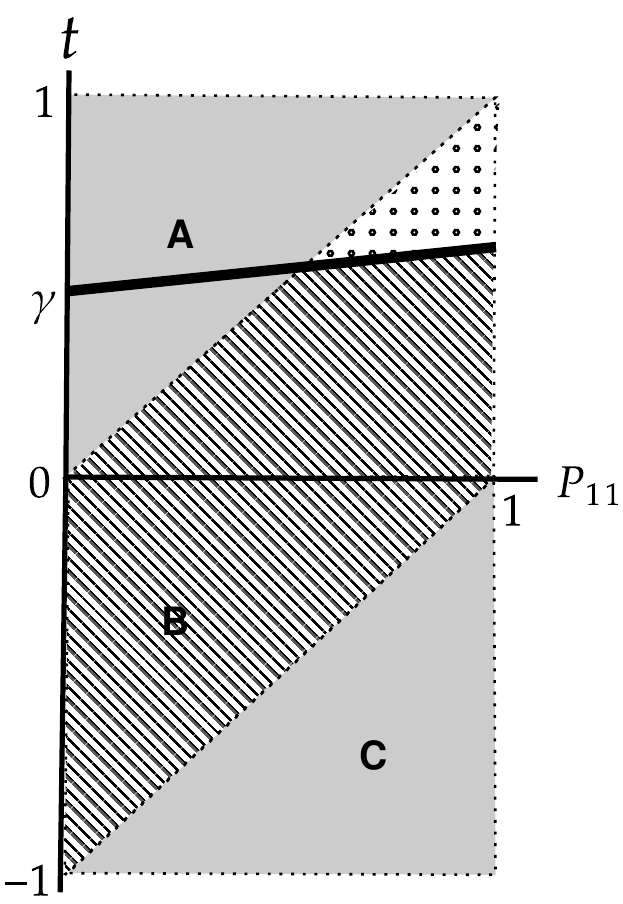}}
 \qquad
 \subfigure[$\delta \in \mathbb{R}^{-}$]{\includegraphics[width=0.293\textwidth]{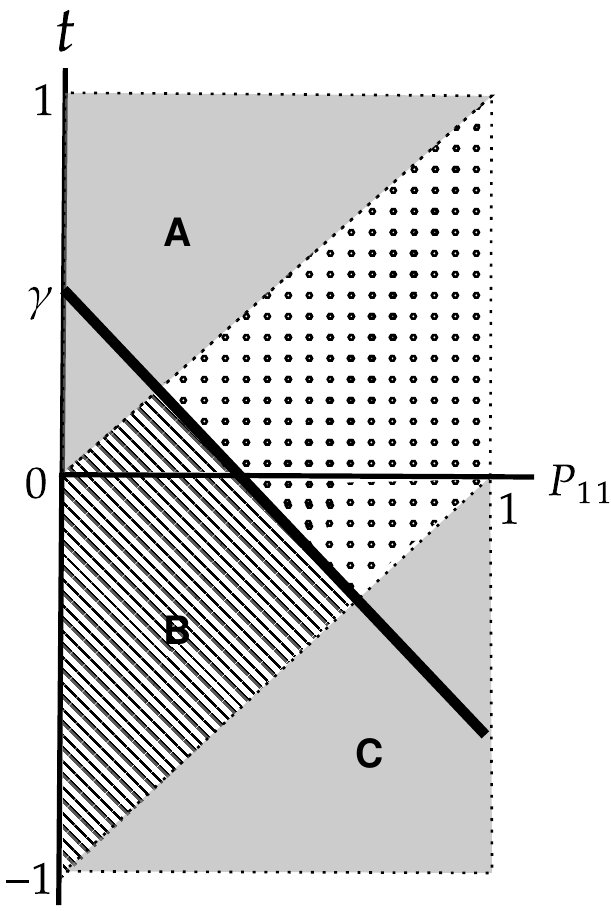}}
 \qquad
 \caption{The cost-sensitive causal classification decision boundary (solid line) with a positive (panel a) or negative (panel b) slope.}
 \label{fig:frontier}
\end{figure}

The cost-sensitive causal classification decision boundary is visualized in panels (a) and (b) of Figure \ref{fig:frontier}, for arbitrary values of the intercept $\gamma$ and a positive (left panel) or negative (right panel) value of the slope, $\delta$. In Section \ref{subsec:csdbanalysis} below, we will further investigate the sign and range of both $\gamma$ and $\delta$ under a set of assumptions for the cost and benefit parameters, reflecting real-world conditions in practical applications. First, we prove that Equation \ref{eq:ITE_threshold} generalizes upon the cost-insensitive causal classification decision boundary, as well as the cost-sensitive classification thresholds specified in Equation \ref{eq:cost_theoretical_threshold}, which itself entails the cost-insensitive classification threshold of Equation \ref{eq:costinsensitiveclassification}.

\begin{proposition}\label{prop:costinsenscausal}
The cost-insensitive causal classification threshold, $\tau_{ci}^*=0$, is obtained for $\mathbf{OB}=0$ and $\mathbf{TC}=0$.
\end{proposition}
The proof of Proposition \ref{prop:costinsenscausal} follows from Equations \ref{eq:gamma} and \ref{eq:delta}, with $\mathbf{OB}=0$ and $\mathbf{TC}=0$.

\begin{theorem}\label{theorem:csthreshold}
The cost-sensitive classification threshold, defined by Equation \ref{eq:cost_theoretical_threshold}, is a specific case of the cost-sensitive causal classification decision boundary, defined by Equation \ref{eq:ITE_threshold}.
\end{theorem}

\begin{proof}
In conventional classification, instances are classified in predicted outcome classes, whereas in causal classification they are classified in treatment classes. In conventional classification, the same treatment is applied to both predicted classes, i.e., $W=0$. 
Hence, a positive outcome value is equally likely, regardless of the treatment or predicted class, i.e., $P_{11}=P_{10}$ and $P_{01}=P_{00}$. 
As a result, $t = P_{11} - P_{10} = 0$ for all instances.

For $t=0$ and rearranging the terms in Equation \ref{eq:ITE_threshold}, we obtain the following equation for the cost-sensitive causal classification decision boundary in terms of the minimum required positive outcome probability, $P^{*}_{11}$, for an instance to be classified in the positive class:

\begin{align*}
    P^{*}_{11} &= -\frac{\gamma}{\delta}, \\ 
    &=  - \frac{b_{00}-c_{00}-b_{01}+c_{01}}{b_{10}-c_{10}+b_{01}-c_{01}-b_{11}+c_{11}-b_{00}+c_{00}},  \\
    &= \frac{\dot{cb_{00}}-\dot{cb_{01}}}{\dot{cb_{11}}+\dot{cb_{00}}-\dot{cb_{01}}-\dot{cb_{10}}}.
\end{align*}

Hence, with $\mathbf{\dot{CB}} = \mathbf{CB}$ and $\phi^*_{cs}$ as defined in Equation \ref{eq:cost_theoretical_threshold} of the cost-sensitive threshold for classification tasks, we find:
\begin{equation}
    P^{*}_{11} = \phi^*_{cs}.
\end{equation}
\end{proof}

Here, the advantage of the notation that was introduced in Section \ref{sec:Prior} for the cost-benefit matrix in conventional classification becomes apparent, which is unusual in the field of classification but aligns with the specification of the causal cost-benefit matrix. 

The importance of Theorem \ref{theorem:csthreshold} is substantial since it adds to the establishment of conventional classification as a specific case of causal classification, as it has been positioned in \cite{verbeke2020foundations} in terms of evaluation measures. As a result, the cost-benefit matrix for conventional classification is to be regarded as the difference between a cost and benefit matrix, as is generally the case in causal classification.

\subsection{Analysis of the cost-sensitive decision boundary}\label{subsec:csdbanalysis}
In this section, we present an analysis of the cost-sensitive causal decision boundary. To this end, we introduce the feasible set, the positive treatment set and the negative treatment set.

\begin{definition}\label{def:feasset}
The \textbf{feasible set}, $\mathcal{F}$ is defined as the set of feasible combinations of the positive outcome probability for the positive treatment, $P_{11}$, and the individual treatment effect, $t$: 
\begin{equation}
    \mathcal{F} = \{(P_{11},t) | P_{11} \in [0,1], P_{11} - 1 \leq t \leq P_{11}\}
\end{equation}
\end{definition}
The feasible set is visualized in panels (a) and (b) of Figure \ref{fig:frontier} by the striped area B and the dotted area. It is bounded by $P_{11}=0$ on the left and $P_{11}=1$ on the right, and by the diagonals $t = P_{11}$ from above and $t = -1 + P_{11}$ from below. The gray areas A and C correspond to infeasible combinations of $t$ and $P_{11}$. 

\begin{definition}\label{def:postreatset}
The \textbf{positive treatment set}, $\mathcal{P}$, is defined as the subset of the feasible set, $\mathcal{F}$,  with $t > \tau^*_{cs}$.
\end{definition}

The positive treatment set corresponds to the dotted area in panels (a) and (b) of Figure \ref{fig:frontier}. 
The positive treatment set contains all combinations of $P_{11},t$ that are to be causally classified in the positive treatment class, $W=1$, under the cost-benefit structure $\mathbf{\dot{CB}}$ and with the objective of maximizing the expected causal profit. 

\begin{definition}\label{def:negtreatset}
The \textbf{negative treatment set}, $\mathcal{N}$, is defined as the subset of the feasible set, $\mathcal{F}$,  with $t \leq \tau^*_{cs}$.
\end{definition}

The negative treatment set corresponds to the striped area B in panels (a) and (b) of Figure \ref{fig:frontier}.

The negative treatment set contains all cases with $t \leq \tau^*_{cs}$, i.e., combinations of $t$ and $P_{11}$ for which the optimal causal classification is $W=0$. For instances in this area, the expected benefit does not outweigh the expected cost of applying the positive treatment. 

\begin{proposition}\label{prop:feasiblesetunion}
The \textbf{feasible set}, $\mathcal{F}$ is the union of the positive and negative treatment set:
    \begin{equation}
        \mathcal{F} = \mathcal{P} \cup \mathcal{N}
    \end{equation}
\end{proposition}

The proof of Proposition \ref{prop:feasiblesetunion} follows directly from Definitions \ref{def:feasset}, \ref{def:postreatset} and \ref{def:negtreatset}.

To accommodate the analysis of the decision boundary and without loss of generality, we simplify the expression of the intercept, $\gamma$, and the slope, $\delta$, by setting the cost of the negative treatment and the benefit of a negative outcome equal to zero, $c_{00}=c_{10}=0$ and $b_{00} = b_{01} = 0$, and expressing the cost of a positive treatment and the benefit of a positive outcome relative to a negative outcome when the negative treatment is applied. This reflects an assessment of the increase in profit that results from applying the positive treatment for a selected subset of instances, compared to the profit that results from applying the negative or no treatment to all instances, which is exactly the definition of the causal profit. Hence, in terms of impact on causal profit, $\gamma$ and $\delta$, the effect of this operation on causal profit is null. We obtain the following causal cost-benefit matrix, with $\dot{cb}_{00} = 0$, i.e., the baseline outcome yielding a net profit equal to zero:

\begin{equation}
    \begin{blockarray}{cccccc}
        & \BAmulticolumn{2}{c}{\text{Treatment}} & & \\
        \addlinespace
        & W=0 & W=1 & & \\
        \addlinespace
        \begin{block}{c[cc]ccc}
             \topstrut \multirow{2}{*}{$\mathbf{\dot{CB}}:=$}  & 0 & -c_{01} & Y=0 & \multirow{2}{*}{Outcome}\\
            \addlinespace
            & \botstrut b_{10} & b_{11}-c_{11} & Y=1\\
      \end{block}
    \end{blockarray}
\end{equation}


As a result, parameter $\gamma$ simplifies to:
\begin{equation}\label{eq:simplegamma}
\gamma' = \frac{c_{01}}{b_{10}}.
\end{equation}
\begin{lemma}\label{lemma:gammapos}
The intercept of the cost-sensitive causal decision boundary is nonnegative:
\begin{equation}
    \gamma' \in \mathbb{R}_{\geq 0}
\end{equation}
\end{lemma}
\begin{proof}
$\gamma'$ is defined as the ratio of $c_{01}$ and $b_{10}$, which by definition, are nonnegative numbers. 
\end{proof}

On the other hand, parameter $\delta$ simplifies to:
\begin{equation}\label{eq:simpledelta}
\delta' = \frac{b_{10} - b_{11} + c_{11} - c_{01}}{b_{10}}.     
\end{equation}

For characterizing the cost-sensitive causal classification decision boundary, we investigate whether the positive treatment set can be empty and whether it can contain instances with a negative estimate for the individual treatment effect. 

\begin{theorem}\label{theorem:postreatempty}
The positive treatment set is nonempty if and only if the causal profitability condition is fulfilled, i.e., if the benefit of a positive outcome given that the positive treatment was applied is larger than the cost of applying the positive treatment when a positive outcome is obtained, $c_{11} < b_{11}$. 
\end{theorem}

\begin{proof}
Visually, for the positive treatment set to be empty, the decision boundary would be fully above the upper bound of the feasible set, i.e., the diagonal $t = P_{11}$.

Since the intercept of the decision boundary is nonnegative, $\gamma' \in \mathbb{R}_{\geq 0}$, as established by Lemma \ref{lemma:gammapos}, the positive treatment set is nonempty if and only if $\tau^*_{cs} < 1$ for $P_{11}=1$.

Elaborating Equation \ref{eq:ITE_threshold}, for $P_{11}=1$ and replacing $\gamma'$ with Equation \ref{eq:simplegamma} and $\delta'$ with Equation \ref{eq:simpledelta}:
\begin{align*}
    \tau^*_{cs}  &= \gamma'+\delta' P_{11} \\
            &= \gamma'+\delta', \\
            &= \frac{c_{01}}{b_{10}} + \frac{b_{10} - b_{11} + c_{11} - c_{01}}{b_{10}},\\
            &= \frac{b_{10} - b_{11} + c_{11}}{b_{10}},
\end{align*}
which is smaller than one if and only if the numerator is smaller than the denominator, or:
\begin{align*}
    b_{10} - b_{11} + c_{11} < b_{10}, \\
    - b_{11} + c_{11} < 0, \\
    c_{11} < b_{11}.
\end{align*}
which is denominated the causal profitability condition. 
\end{proof}

From the above proof, note that 
in the case of a positive outcome, the closer the cost of the positive treatment ($c_{11}$) approaches the benefit of a positive outcome ($b_{11}$), the closer will $\tau^*_{cs}$ approach one. Hence, the smaller the net return of a positive converted instance is, the larger the required estimated ITE for the expected return to be positive and for classifying an instance in the positive treatment class. Intuitively, it makes sense that the positive treatment should be applied to fewer instances when the net profit of a positive converted instance is smaller. Visually, the area representing the positive treatment set becomes smaller for a larger slope of the decision boundary, for any (positive) value of the intercept.

As a corollary of Theorem \ref{theorem:postreatempty}, we have:
\begin{corollary}\label{theorem:postreatemptyinv}
The positive treatment set is empty if the cost of a positive converted instance is larger than the benefit of a positive converted instance:
If $c_{11} > b_{11}$ then $\mathcal{P} = \emptyset$. 
\end{corollary}
If the causal profitability condition is not satisfied, then a positive converted instance yields a loss and no causal profit can be obtained from applying the positive treatment.

\begin{theorem}\label{theorem:postreatpos}
The positive treatment set contains pairs $(P_{11},t)$ with $t<0$, if and only if the causal bonus condition is satisfied, i.e., if the cost of a positive treatment for a positive outcome is smaller than the difference between the benefit of a positive outcome for a positive and a negative treatment (i.e., the bonus), $c_{11} < b_{11}-b_{10}$. 
\end{theorem}

\begin{proof}
We elaborate upon the proof of Theorem \ref{theorem:postreatempty} since, given the intercept is positive as established in Lemma \ref{lemma:gammapos}, the positive treatment set will only contain pairs $(P_{11},t)$ with $t<0$ if and only if $\tau^*_{cs} < 0$ for $P_{11}=1$.
We found that:
\begin{equation*}
    \tau^*_{cs} = \frac{b_{10} - b_{11} + c_{11}}{b_{10}},
\end{equation*}
which, since $b_{10}$ by definition is nonnegative, is smaller than zero if and only if:
\begin{align*}
    b_{10} - b_{11} + c_{11} < 0, \\
    c_{11} < b_{11} - b_{10},
\end{align*}
which is denominated the causal bonus condition. 
\end{proof}
As a corollary of Theorem \ref{theorem:postreatpos}, we have:
\begin{corollary}\label{cor:posITE}
The positive treatment set cannot contain pairs $(P_{11},t)$ with $t<0$ if the causal bonus condition is not satisfied, i.e., if $c_{11} \geq b_{11}-b_{10}$, which includes all cases with $b_{11} \leq b_{10}$.
\end{corollary}

Theorem \ref{theorem:postreatpos} and the causal bonus condition may be intuitively difficult to accept and interpret since it is unexpected that the positive treatment may have to be applied even when the estimated ITE is negative. 

If the causal bonus condition is satisfied, then for an instance with a positive outcome for the negative treatment, we may still have to apply the positive treatment, not so much for changing the outcome from the negative to the positive class, but to reap a potential bonus since the increase in benefit for a positive outcome and positive treatment, $b_{11}$, compared to a positive outcome for a negative treatment, $b_{10}$, is larger than the cost of the positive treatment, $c_{11}$. 
We are not guaranteed to reap the bonus even if we are sure the outcome for the negative treatment is positive, since applying the positive treatment may have a negative effect on the outcome. For a negative converted instance, the positive outcome for negative treatment converts into a negative outcome for positive treatment. 
Theorem \ref{theorem:postreatpos} indicates that if there is a potential bonus of applying the positive treatment, even when the probability for the positive outcome probability decreases (to some extent) as a result of applying the treatment, i.e., for a negative estimated ITE, the positive treatment may still have to be applied. It is the causal cost-sensitive decision boundary that balances the risk of a negative conversion and the opportunity of seizing a bonus.

In many applications, the benefit of a positive outcome may be expected to be independent of the treatment that has been applied, i.e., $b_{11} = b_{10}$. Then, as well as for $b_{11} < b_{10}$ and generally  $c_{11} \geq b_{11}-b_{10}$, Corollary \ref{cor:posITE} indicates that no instances with negative estimated ITE will be classified in the positive treatment class.

\section{Cost-sensitive ranking}

\subsection{Rationale}
In the case of a limited number of available positive treatments that can be applied because of resource restrictions, such as a limitation with respect to the available budget, time, and human or natural resources, the optimal cost-sensitive decision boundary as defined by Equation \ref{eq:ITE_threshold} may not be feasible to implement. 

A restriction on the number of available treatments represents a realistic scenario in many  application settings and does not necessarily implies a global suboptimal operating point, since lifting the restriction may alter the cost-benefit structure of the problem. For instance, in order to apply more positive treatments, an investment may be required in infrastructure, thus significantly altering the cost per treatment. We enlist a deeper analysis of problem settings with resource-dependent treatment-cost and outcome-benefit matrices as a topic for further research. 

Given the applicable restriction, the question arises as to how to classify a subset of instances in the positive treatment class in a cost-sensitive manner? Which instances are to be prioritized in applying the limited number of positive treatments that are available? 

Intuitively, it makes sense to shift the decision boundary upward by increasing the intercept of the decision boundary while keeping the slope constant, until the number of instances above the decision boundary (i.e., in the positive treatment class) equals the number of positive treatments that can be applied, i.e., until the constraint is satisfied. This would be the cost-sensitive equivalent of the cost-insensitive approach of ranking instances based on the ITE estimates to support the selection of the optimal subset of instances. 

\subsection{Expected causal profit ranker}
Mathematically, this problem involves a constrained optimization problem. The reward function that is to be maximized is the expected causal profit, as defined in Definition \ref{def:expectedcausalprofit}. Hence, instances are to be ranked and selected based on the expected causal profit. This, as we will prove below, is equivalent too the intuitive approach of ranking instances based on the distance to the decision boundary, or better, the displacement from the decision boundary to the instance, which is negative for instances in the positive treatment set and negative for instances in the negative treatment set and thus yields an appropriate ranking.

The displacement $d(x_i)$ from the decision boundary $\tau^*_{cs} = \gamma + \delta P_{11}$ for instance $x_i$ with coordinates $\big(P_{11,i},t(x_i)\big)$ is given by:

\begin{equation} \label{eq:ITE_distance}
    d(x_i) = \frac{t(x_i) - \delta P_{11,i} - \gamma}{\sqrt{\delta^2 + 1}}
\end{equation}

\begin{proposition}\label{theorem:distexcausprof}
The displacement from the decision boundary for an instance, as defined in Equation \ref{eq:ITE_distance}, is proportional to the expected causal profit of the instance, as defined by Equation \ref{eq:expectedcausalprofit}. 
\end{proposition}

\begin{proof}
The expected causal profit is defined in Equation \ref{eq:expectedcausalprofit} as:
\begin{align*}
    E[\dot{P}|x] = E[P|x,1] - E[P|x,0].
\end{align*}
By entering the formula for the expected profit, $E[P|x,w]$, as defined in Equation \ref{eq:expectedprofit}, and by elaborating the resulting equation (the elaboration is similar to that of obtaining Equation \ref{eq:ITE_threshold} of the cost-sensitive causal decision boundary), we obtain:
\begin{align}
    E[\dot{P}|x] &=  P_{11} (b_{11}-c_{11}) + (1-P_{11}) (b_{01}-c_{01}) - P_{10} (b_{10}-c_{10}) - (1-P_{10}) (b_{00}-c_{00}), \nonumber \\
    &=  t(x) - \delta P_{11} - \gamma
\end{align}
Since $\sqrt{\delta^2 + 1}$ is a constant, i.e., a function of the cost and benefit parameters, we find that the expected causal profit, $E[\dot{P}|x]$, is proportional to the displacement $d(x)$ from the cost-sensitive causal decision boundary to the instance, as defined in Equation \ref{eq:ITE_distance}.

\end{proof}

We denominate the proposed ranking approach as the expected causal profit (ECP) ranker, whereas the conventional ranking approach based on the estimated ITE will be denominated the ITE ranker.

The ECP ranker supports decision-makers in selecting the set of instances that should be treated with the positive treatment given a restriction on the number of positive treatments that can be applied. For a restricted number of treatments $r$, the ECP ranker facilitates straightforward selection of the optimal subset, i.e., the $r$ first ranked instances, to maximize the expected causal profit given the constraint. 

In case of a restriction on the available budget, the following straightforward constrained optimization problem is to be solved.  Let $\mathcal{D}=\{x_j\}$ for $j=1 \dots n$ with $d(x_j) \geq d(x_j+1) \geq 0$ be the resulting ranked set of instances and $m$ the budget for applying treatments. Then, the optimal number of treated instances, $T(\mathcal{D})$, can be computed as:

\begin{equation} \label{eq:optimization_problem}
\begin{aligned}
\max \quad & T(\mathcal{D}) = T_{0} + T_{1}\\
\textrm{s.t.} \quad &  T_{0} c_{01} + T_{1} c_{11}\leq m.\\
\end{aligned}
\end{equation}

With $T_{0}$ and $T_{1}$ the estimated numbers of negative and positive outcomes, respectively, equal to:

\begin{align}
    T_{0} = \sum_{j=1}^T (1-P_{11,j}), \\
    T_{1} = \sum_{j=1}^T P_{11,j},
\end{align}

with $P_{11,j} = P(Y=1|x_j,w=1)$ for $x_j \in \mathcal{D}$. 

\section{Empirical evaluation}
\label{sec:Experiments}

In this section, the cost-sensitive causal decision-making framework is applied on synthetic and marketing campaign data (see Table \ref{tab:info_datasets}), all publicly available. The main objective in this section is to illustrate and evaluate the use of the newly proposed ECP ranker compared to the ITE ranker. An extensive benchmark experiment and detailed discussion of the cases are beyond the scope of this study. By providing the full experimental code\footnote{The experimental code and data that are used in this study will be published in a GitHub repository upon publication of this article and can be provided to the reviewers upon request.} that we used in arriving at the reported results, we aim to promote the adoption of the presented framework and the ECP ranker and facilitate the validation of the presented results across various decision-making settings and the publication of additional empirical results.

We first introduce the data and then outline the methodology, including the experimental design, the cost-benefit parameters, the causal classification methods and the performance metrics. Finally, we present and discuss the experimental results.

\subsection{Data}

\begin{table} [ht]
\centering
\caption{Data sets information}
\label{tab:info_datasets} 
\resizebox{370pt}{!}{%
\begin{tabular}{ l c c c c}
\noalign{\smallskip}\toprule
\multicolumn{3}{l}{Data}\\ \hline
\noalign{\smallskip}
Organization & Synthetic & Criteo & Bank & Hillstrom\\ \hline
\noalign{\smallskip}
Total number of observations & 10000 & 11647 & 10043 & 14232\\
\noalign{\smallskip}
Total number of variables & 16 & 12 & 160 & 17\\
\noalign{\smallskip}
Number of control group observations & 5000 & 11292 & 5939 & 7102\\
\noalign{\smallskip}
Proportion of positive outcomes in control group & 0.49 & 0.04 & 0.25 & 0.10\\
\noalign{\smallskip}
Number of treatment group observations & 5000 & 355 & 4104 & 7130\\
\noalign{\smallskip}
Proportion of positive outcomes in treatment group & 0.57 & 0.41 & 0.13 & 0.15\\
\noalign{\smallskip}
Overall effect & 0.08 & 0.37 & 0.12 & 0.04\\
\bottomrule
\end{tabular}}
\end{table}

We employed the \textit{causalml} library \cite{chen2020causalml} to generate a synthetic causal data set, which is based on the approach developed by \cite{guyon2003design} to generate the Madelon data set. The treatment and control groups are specified to be of equal size and the positive class proportion is set at 50\%. The data set consists of 16 pretreatment variables among which 11 are set to generate positive treatment effects and 5 are designated for base classification.

The Criteo data set is a large publicly available data set \cite{diemert2018large} from which we sampled a random subset, maintaining the class distribution in both the treatment and the control group, so that the overall treatment effect is the same as in the original data set. The Criteo data set involves a randomized controlled trial that is conducted to analyze the effect of a digital advertising campaign. Subjects that were exposed to the campaign constitute the treatment group (positive treatment, $W=1$), whereas control group instances were not targeted (negative treatment, $W=0$). The data set contains 12 pretreatment variables, the treatment indicator and the binary outcome variable. A positive outcome ($Y=1$) represents a user who visited the advertiser website during a test period of 2 weeks. The positive outcome rate in the control group is 37\% and in the treatment group 41\%, yielding an overall positive effect of the campaign of 4\%. 

The Bank data set is obtained from a financial institution that applied a retention campaign to reduce customer churn. This data set has been used in prior studies; for additional details and experimental results one may refer to \cite{devriendt2019you}. Customers that were targeted with the retention campaign constitute the treatment group (positive treatment, $W=1$), whereas the customers in the control group were not targeted (negative treatment, $W=0$). The outcome variable indicates whether a customer churned in the period after the campaign was deployed, which is the negative outcome ($Y=0$), or whether a customer did not churn, which is the positive outcome ($Y=1$). Pretreatment variables include sociodemographic information, as well as usage and activity indicators. The churn rate in the control group is equal to 25\% and in the treatment group equal to 13\%, yielding an overall positive effect of the campaign of 12\%.

The Hillstrom data set consists of data on a direct marketing campaign \cite{hillstrom2008minethatdata}. The treatment group comprises customers who received an e-mail with women's merchandise, whereas the control group are customers who were not contacted. The outcome variable indicates whether a customer visits the website within two weeks after the campaign. We use a subset of the original data set consisting of 22\% randomly subsampled instances, maintaining the original class and group distribution as well as the overall positive effect of the campaign on the positive outcome rate.

\subsection{Methodology}

We consider two cost-benefit scenarios to illustrate the use and assess the performance of the ECP ranker. In the first scenario, the benefit of a positive outcome for $W=1$ is assumed to be 1.2 times the benefit of a positive outcome for $W=0$, i.e., $b_{11} > b_{10}$. In the second scenario, the benefit of a positive outcome for $W=0$ is set to 1.2 times the benefit of a positive outcome for $W=1$, i.e., $b_{10} > b_{11}$. On the other hand, the cost of the treatment is the same for a positive and negative outcome, i.e., $c_{01} = c_{11}$. No benefit is obtained in case of a negative outcome, and the negative treatment does not generate any costs. 

The value of the baseline benefit is set to $100$ monetary units, i.e., $b_{11} = 120$ and $b_{10} = 100$ in the first scenario and $b_{11} = 100$ and $b_{10} = 120$, whereas the cost of a positive treatment is fixed at $10\%$ of the baseline benefit, i.e., $c_{01} = c_{11} = 10$ monetary units.

Three causal classification metalearners, i.e., the T- and S-learner as well as the causal dorest (CF) method, are applied in combination with the base classifiers logistic regression (LR) and extreme gradient boosting (XGBoost) \cite{chen2016xgboost} to estimate individual treatment effects. 

Logistic regression is a typical baseline method that is included in experimental benchmarking studies to assess the performance of machine learning methods for business decision-making. It offers a good balance between interpretability, stability and predictive power and is therefore broadly adopted in industry across various fields of application \cite{devriendt2019you}. On the other hand, ensemble methods such as XGBoost have been shown to potentially achieve (much) better predictive power across different studies and competitions \cite{le2020xgboost}, an amenity that comes at the cost of reduced interpretability of the resulting model. 

The T- and S-learners are well-known baseline techniques for individual treatment effect estimation. Both metalearners are implemented as indicated in \cite{kunzel2019metalearners}. The T-learner is the most intuitive method and has demonstrated competitive performance under certain conditions \cite{rudas2018linear}. The causal forest (CF) \cite{wager2018estimation,athey2019generalized} is a forest-based method that has been reported to achieve good performance in prior studies \cite{davis2017using,luo2019and}. We use the CF implemented by \cite{econml} and specify the LR and XGBoost as estimators to fit the outcome and the treatment to the pretreatment variables.

In addition to the various combinations of the metalearners and base learners for estimating the individual treatment effect, we apply both the ITE and ECP ranker, as introduced in this article.

For evaluating the performance of the resulting causal classification models, we adopt the Qini coefficient \cite{radcliffe2007using}, which is the conventional, cost-insensitive approach in the literature to assess the performance of causal classification models. In addition, we evaluate performance in a cost-sensitive manner in terms of the average causal profit, $(\dot{AP})$, and the maximum causal profit, $(\dot{MP})$, measures, as defined in Section \ref{sec:classification}. Finally, we apply five-fold cross-validation, which is a standard procedure for achieving a reliable performance assessment.

\subsection{Results}

Table \ref{tab:qinis} reports the Qini coefficient for the four data sets, for the various combinations of ranker, metalearners and base learners and for the two cost-benefit scenarios. Note that a larger value of the Qini coefficient indicates an overall (across thresholds) better performance of the causal model in terms of increasing the positive outcome rate. The result of the best model is shown in bold for each data set and the result of the best ranking approach, cost-insensitive or cost-sensitive, is underlined. 

In Table \ref{tab:qinis}, we observe that the ITE ranker achieves the overall best performance (in bold) in five setups, whereas the ECP ranker in three setups. 
The ECP ranker does not aim to improve performance in Qini but in terms of profit; however, it is observed to do so across the three marketing campaign data sets, specifically for the scenario $b_{11} > b_{10}$. 
For the scenario $b_{11} < b_{10}$, on the other hand, it is the ITE ranker that achieves the best performance across all data sets.
When assessing the best performance per meta- and base-learner combination, we almost consistently observe the same effect for all three marketing campaign data sets, with the ECP ranker outperforming the ITE ranker for the scenario $b_{11} > b_{10}$ and vice versa for $b_{11} < b_{10}$. On the synthetic data set, however, the ITE ranker outperforms the ECP ranker consistently for both scenarios, with the exception of the LR-S-Learner combination.

\begin{table} [ht]
\centering
\caption{Qini coefficient}
\label{tab:qinis} 
\resizebox{420pt}{!}{%
\begin{tabular}{ l c c c c c c c c c c c c c c c c c }
\toprule \noalign{\smallskip}
\multirow{5}{*}{Model} & \multirow{5}{*}{Benefit} & \multicolumn{4}{c}{Synthetic} &\multicolumn{4}{c}{Bank} & \multicolumn{4}{c}{Criteo} & \multicolumn{4}{c}{Hillstrom}\\
 \noalign{\smallskip} \cmidrule(lr){3-6} \cmidrule(lr){7-10} \cmidrule(lr){11-14} \cmidrule(lr){15-18}\noalign{\smallskip}
 &  & \multicolumn{2}{c}{LR} & \multicolumn{2}{c}{XGBoost} & \multicolumn{2}{c}{LR} & \multicolumn{2}{c}{XGBoost} & \multicolumn{2}{c}{LR} & \multicolumn{2}{c}{XGBoost} & \multicolumn{2}{c}{LR} & \multicolumn{2}{c}{XGBoost}\\ 
\noalign{\smallskip} \cmidrule(lr){3-4}\cmidrule(lr){5-6}\cmidrule(lr){7-8}\cmidrule(lr){9-10}\cmidrule(lr){11-12}\cmidrule(lr){13-14}\cmidrule(lr){15-16}\cmidrule(lr){17-18} \noalign{\smallskip}
 & & ITE & ECP & ITE & ECP & ITE & ECP & ITE & ECP & ITE & ECP & ITE & ECP & ITE & ECP & ITE & ECP\\ 
 \noalign{\smallskip} \hline \noalign{\smallskip}
\multirow{2}{*}{CF} & $b_{11} > b_{10}$ & \underline{0.27} & 0.25	& \underline{0.28} & 0.27 & -0.09 & \underline{-0.04} & \underline{-0.02} & -0.07 & 0.06 & \underline{0.07} & 0.05 & \underline{0.06} & 0.02 & \underline{0.04} & -0.01 & \underline{0.02}\\
\noalign{\smallskip}
 & $b_{11} < b_{10}$ & \underline{0.27} & 0.24 & \underline{0.28} & 0.25 & \underline{-0.09} & -0.13 & \underline{-0.02} & -0.05 & \underline{0.06} & 0.04 & \underline{0.05} & 0.03 & \underline{0.02} & -0.01 & \underline{-0.01} & -0.04\\
\noalign{\smallskip} \cmidrule(lr){1-18} \noalign{\smallskip}
\multirow{2}{*}{T-learner} & $b_{11} > b_{10}$ & \textbf{0.33} & 0.32 & \underline{0.31} & 0.29 & 0.02 & \underline{0.03} & 0.12 & \textbf{0.13} & 0.05 & \underline{0.07} & 0.01 & \underline{0.02} & 0.03 & \underline{0.05} & 0.02 & \underline{0.04}\\
\noalign{\smallskip}
&  $b_{11} < b_{10}$ & \textbf{0.33} & 0.31 & \underline{0.31} & 0.29 & \underline{0.02} & 0.01 & \textbf{0.12} & 0.09 & \underline{0.05} & 0.03 & \underline{0.01} & 0.004 & \textbf{0.03} & 0.01 &	\underline{0.02} & 0.01\\
\noalign{\smallskip} \cmidrule(lr){1-18} \noalign{\smallskip}
\multirow{2}{*}{S-learner} & $b_{11} > b_{10}$ & 0.14	& \underline{0.15} & \underline{0.28} & 0.27 & 0.02 & \underline{0.03} & -0.04 & \underline{0.01} & 0.07 & \textbf{0.08} & 0.04 & \underline{0.06} & \underline{0.02} & 0.01 & 0.02 & \textbf{0.06}\\
\noalign{\smallskip}
& $b_{11} < b_{10}$ & 0.14 & \underline{0.15} & \underline{0.28} & 0.26 & \underline{0.02} & 0.01 & \underline{-0.04} & -0.08 & \textbf{0.07} & 0.06 & \underline{0.04} & 0.02 & \underline{0.02} & -0.01 & \underline{0.02} & 0.01\\
\noalign{\smallskip}
\bottomrule
\end{tabular}}
\end{table}

Table \ref{tab:average_profit_benefit} reports the average causal profit per instance. For the scenario with $b_{11} > b_{10}$, the ECP ranker achieves the overall best performance (in bold) across all data sets. For $b_{11} < b_{10}$, however, the ITE ranker performs best for the three campaign data sets. When assessing performance per metalearner and base learner combination, we observe that for the scenario $b_{11} > b_{10}$, the ECP ranker without exception improves upon the ITE ranker, which is a notable result. 

For $b_{11} < b_{10}$, except for the synthetic data set and the combination XGBoost-S-learner on the Hillstrom data set, we see that the ITE ranker performs best, which is unexpected. A further analysis of this result is provided in Section \ref{subsec:analysisresults}.

Note that the absolute difference in performance between the ITE and ECP ranker that is observed in Table \ref{tab:average_profit_benefit} results from the arbitrary values that have been used in the causal cost-benefit matrix.

\begin{table} [ht]
\centering
\caption{Average causal profit, $\dot{AP}$.}
\label{tab:average_profit_benefit} 
\resizebox{460pt}{!}{%
\begin{tabular}{ l c c c c c c c c c c c c c c c c c}
\toprule \noalign{\smallskip}
\multirow{5}{*}{Model} & \multirow{5}{*}{Benefit} & \multicolumn{4}{c}{Synthetic} &\multicolumn{4}{c}{Bank} & \multicolumn{4}{c}{Criteo} & \multicolumn{4}{c}{Hillstrom}\\
 \noalign{\smallskip} \cmidrule(lr){3-6} \cmidrule(lr){7-10} \cmidrule(lr){11-14} \cmidrule(lr){15-18}\noalign{\smallskip}
 &  & \multicolumn{2}{c}{LR} & \multicolumn{2}{c}{XGBoost} & \multicolumn{2}{c}{LR} & \multicolumn{2}{c}{XGBoost} & \multicolumn{2}{c}{LR} & \multicolumn{2}{c}{XGBoost} & \multicolumn{2}{c}{LR} & \multicolumn{2}{c}{XGBoost}\\ 
\noalign{\smallskip} \cmidrule(lr){3-4}\cmidrule(lr){5-6}\cmidrule(lr){7-8}\cmidrule(lr){9-10}\cmidrule(lr){11-12}\cmidrule(lr){13-14}\cmidrule(lr){15-16}\cmidrule(lr){17-18} \noalign{\smallskip}
 & & ITE & ECP & ITE & ECP & ITE & ECP & ITE & ECP & ITE & ECP & ITE & ECP & ITE & ECP & ITE & ECP\\  
 \noalign{\smallskip} \hline \noalign{\smallskip}
\multirow{2}{*}{CF} & $b_{11} > b_{10}$ & 4.55 & \underline{6.34} & 4.66 & \underline{6.42} & -6.22 & \underline{-2.82} & 3.17 & \underline{9.76} & 27.83 & \underline{28.29} & 27.28 & \underline{28.59} & -1.98 & \underline{-1.62} & -2.26 & \underline{-1.86}\\
\noalign{\smallskip}
 & $b_{11} < b_{10}$ & -1.33 & \underline{-0.14} & -1.16 & \underline{-0.02} & \underline{-20.40} & -21.25 & \underline{-11.66} & -14.16 & \underline{22.24} & 19.47 & \underline{21.80} & 19.00 & \underline{-3.68} & -3.82 & \underline{-3.91} &	-4.03\\
\noalign{\smallskip} \cmidrule(lr){1-18} \noalign{\smallskip}
\multirow{2}{*}{T-learner} & $b_{11} > b_{10}$ & 3.95 & \underline{6.28} & 3.93 & \textbf{6.52} & 7.44 & \underline{9.21} & 9.96 & \textbf{12.45} & 26.75 & \underline{29.98} & 20.39 & \underline{21.71} & -1.82 & \underline{-1.52} & -2.03 & \underline{-1.73}\\
\noalign{\smallskip}
&  $b_{11} < b_{10}$ & 0.63 & \textbf{0.70} & 0.25 & \underline{0.58} & \underline{-7.37} & -8.63 & \textbf{-4.33} & -6.05 & \underline{21.34} & 18.32 & \underline{16.21} & 14.68 & \textbf{-3.55} & -3.71 & \underline{-3.61} & -3.70\\
\noalign{\smallskip} \cmidrule(lr){1-18} \noalign{\smallskip}
\multirow{2}{*}{S-learner} & $b_{11} > b_{10}$ & 2.56 & \underline{4.38} & 3.36 & \underline{6.41} & 7.44 & \underline{9.21} & 2.03 & \underline{10.00} & 30.59 & \textbf{31.03} & 25.32 & \underline{27.78} & -2.10 & \underline{-1.92} & -1.95 & \textbf{-1.47}\\
\noalign{\smallskip}
& $b_{11} < b_{10}$ & -2.75 & \underline{-1.47} & 0.21 & \underline{0.23} & \underline{-7.37} & -8.63 & \underline{-11.46} & -13.28 & \textbf{24.35} & 22.77 & \underline{20.16} & 16.62 & \textbf{-3.56} & -3.71 & -3.65 & \underline{-3.62}\\
\noalign{\smallskip}
\bottomrule
\end{tabular}}
\end{table}

Finally, the maximum causal profit is reported in Table \ref{tab:max_profit_benefit}, together with the optimal proportion of instances in the positive treatment class. The results are, as expected, almost fully in line with the results in Table \ref{tab:average_profit_benefit}.

\begin{table} [ht]
\centering
\caption{Maximum profit, $\dot{MP}$, for varying levels of the benefit, $b$}
\label{tab:max_profit_benefit} 
\resizebox{420pt}{!}{%
\begin{tabular}{ l c c c c c c c c c c c c c c c c c}
\toprule \noalign{\smallskip}
\multirow{5}{*}{Model} & \multirow{5}{*}{Benefit} & \multicolumn{4}{c}{Synthetic} &\multicolumn{4}{c}{Bank} & \multicolumn{4}{c}{Criteo} & \multicolumn{4}{c}{Hillstrom}\\
 \noalign{\smallskip} \cmidrule(lr){3-6} \cmidrule(lr){7-10} \cmidrule(lr){11-14} \cmidrule(lr){15-18}\noalign{\smallskip}
 &  & \multicolumn{2}{c}{LR} & \multicolumn{2}{c}{XGBoost} & \multicolumn{2}{c}{LR} & \multicolumn{2}{c}{XGBoost} & \multicolumn{2}{c}{LR} & \multicolumn{2}{c}{XGBoost} & \multicolumn{2}{c}{LR} & \multicolumn{2}{c}{XGBoost}\\ 
\noalign{\smallskip} \cmidrule(lr){3-4}\cmidrule(lr){5-6}\cmidrule(lr){7-8}\cmidrule(lr){9-10}\cmidrule(lr){11-12}\cmidrule(lr){13-14}\cmidrule(lr){15-16}\cmidrule(lr){17-18} \noalign{\smallskip}
 & & ITE & ECP & ITE & ECP & ITE & ECP & ITE & ECP & ITE & ECP & ITE & ECP & ITE & ECP & ITE & ECP\\ 
 \noalign{\smallskip} \hline \noalign{\smallskip}
\multirow{4}{*}{CF} & \multirow{2}{*}{$b_{11} > b_{10}$} & 6.62 & \underline{9.06} & 6.89 & \underline{9.10} & 17.63 & 17.63 & 17.63 & 17.63 & 34.72 & \textbf{34.99} & 34.42 & \underline{34.59} & 0 & \underline{0.29} & 0.02 & \underline{0.41}\\
& & (16) & (41) & (18) & (42) & (100) & (100) & (100) &	(100) & (98) &  (97) & (95) & (94) & (0) & (9) & (2) & (2)\\
\noalign{\smallskip}
 & \multirow{2}{*}{$b_{11} < b_{10}$} & \underline{4.22} & 3.87 & \underline{4.17} & 3.86 & 0	& 0 & 0.01 & \underline{0.02} & \textbf{27.70} & 27.16 & \underline{27.47} & 27.16 & 0 & 0 & 0 & 0\\
 & & (12) &	(16) & (14) & (17) & (0) & (0) & (5) & (1) & (98) &	(100) &	(95) & (100) & (0) & (0) & (0) & (0)\\
\noalign{\smallskip} \cmidrule(lr){1-18} \noalign{\smallskip}
\multirow{4}{*}{T-learner} & \multirow{2}{*}{$b_{11} > b_{10}$} & 5.72 & \underline{9.05} & 6.58 & \textbf{9.51} & 17.63 & 17.63 & 18.39 & \textbf{19.84} & 34.15 & \underline{34.41} & 34.15 & 34.15 & 0.31 & \underline{0.64} & 0 & 0.19\\
& & (17) & (36) & (14) & (32) & (100) & (100) &	(89) & (82) & (100) & (75) & (100) & (100) & (1) & (4) & (0) & (12)\\
\noalign{\smallskip}
& \multirow{2}{*}{$b_{11} < b_{10}$} & 4.26 & \underline{4.40} & 5.16 & \textbf{5.23} & \textbf{0.76} & 0.69 & \underline{0.50} & 0.38 & 27.16 & 27.16 & 27.16 & 27.16 & \textbf{0.21} & 0 & 0 & 0\\
& & (10) & (9) & (8) & (9) & (8) & (6) & (4) & (4) & (100) & (100) & (100) & (100) & (1) & (0) & (0) & (0)\\
\noalign{\smallskip} \cmidrule(lr){1-18} \noalign{\smallskip}
\multirow{4}{*}{S-learner} & \multirow{2}{*}{$b_{11} > b_{10}$} & 4.02 & \underline{6.69} & 6.49 & \underline{9.24} & 17.63 & 17.63 & 17.63 & \underline{18.04} & 34.16 & 34.16 & \underline{34.48} & 34.41 & 0.21 & \textbf{0.48} & 0.04 & \underline{0.47}\\
& & (26) & (49) & (13) & (29) & (100) & (100) & (100) & (91) & (96) & (96) & (97) & (97) & (7) & (1) & (5) & (3)\\
\noalign{\smallskip}
& \multirow{2}{*}{$b_{11} < b_{10}$} & 0.77 & \underline{1.28} & 5.02 & \underline{5.21}	& \textbf{0.76} & 0.69 & 0 & 0 & \underline{27.17} & 27.16 & \underline{27.45} & 27.16 & 0 & 0 & 0 & 0\\
& & (18) & (22) & (13) & (12) & (8) & (6) & (0) & (0) & (96) & (100) & (97) & (100) & (0) &	(0)	& (0) & (0)\\
\noalign{\smallskip}
\bottomrule
\end{tabular}}
\end{table}

\subsection{Discussion}\label{subsec:analysisresults}

The results for the scenario $b_{11} > b_{10}$ in the previous section on the three marketing campaign data sets indicate an improvement in performance for the ECP ranker in terms of average and maximum causal profit, as well as, unexpectedly, in terms of the Qini measure. The performance does not improve and even declines for all three measures for the scenario $b_{11} < b_{10}$, which is unexpected. In this section, we therefore further explore the results to explain these findings.

To this end, we elaborate Equation \ref{eq:profit}, which defines the causal profit in terms of the causal confusion matrix for $b_{00}=b_{01}=0$, in line with the experimental setup. Moreover, since $b_{11} \approx b_{10} >> c_{11} = c_{01}$, we set $c_{ij}=0$ $\forall i,j$ to facilitate the analysis. As a result, we obtain the following expression for the causal profit as a function of the positive outcomes in the positive and negative treatment class:
\begin{align}\label{eq:causalprofitphi}
\dot{P}(\tau) &= \sum_{i,j}(\mathbf{\dot{E_{ij}} \circ \dot{CB_{ij}}}), \nonumber \\
        &= \dot{e}_{10} b_{10} + \dot{e}_{11} b_{11}, \nonumber \\
        &= \frac{-\big(C_1-C_1(\tau)\big)}{C_1 + C_0} b_{10}  + \frac{\big(T_1-T_1(\tau)\big)}{T_1 + T_0} b_{11}.
\end{align}
with $C_1(\tau)$ and $T_1(\tau)$ the number of positive instances with $t<\tau$, and $C_y$ and $T_y$ the number of instances with $Y=y$ in the treatment and control group, respectively. Hence, $\dot{e}_{10}$ is to be interpreted as the relative number of positives in the positive treatment class in the control data set and $\dot{e}_{11}$ as the relative number of positives in the positive treatment class of the treatment data set for threshold $\tau$. As such, by assessing the effect of the ECP ranker on the relative number of positives in the positive treatment class for both the treatment and control group, across thresholds $\tau$, we may gain a more fundamental insight into the effect of the ranker on the causal profit. 

In Section \ref{sec:proposal}, we discussed the effect of the ECP ranker, which for the scenario $b_{11} > b_{10}$ (i.e., for a negative slope of the decision boundary) aims at increasing the causal profit by increasing the number of positive instances in the positive treatment class, by giving priority to instances with high values of $P_{11}$. Note that the increase in positive instances will take place in both the treatment and control data sets, so both $\dot{e}_{11}$ and $\dot{e}_{10}$ are expected to increase equally. However, as can be seen from Equation \ref{eq:causalprofitphi}, since $b_{11} > b_{10}$, the net effect on the causal profit of the equal increase in both $\dot{e}_{11}$ and $\dot{e}_{10}$ will be positive. 
In contrast, for $b_{11} < b_{10}$ (i.e., for a positive slope of the decision boundary), the ECP ranker prioritizes instances with lower values of $P_{11}$. The decrease in positive instances in the positive treatment class is expected to be approximately the same for the treatment and control data sets, but since $b_{11} < b_{10}$, the net effect of the decrease in $\dot{e}_{10}$ (which has a positive effect on $\dot{P}$) will outweigh the net effect of the decrease in $\dot{e}_{11}$ (which has a negative effect on $\dot{P}$). As a result, the causal profit increases.  

In Figure \ref{fig:cumposplotsa}, the cumulative proportion of instances with a positive outcome is plotted as a function of the threshold $\tau$, both for the ECP (in blue) and ITE (in orange) ranker, and separately for the treatment and control group of the Bank data set.

\begin{figure}[ht]
 \centering
 \subfigure[Bank: treatment.]{\includegraphics[width=0.45\textwidth]{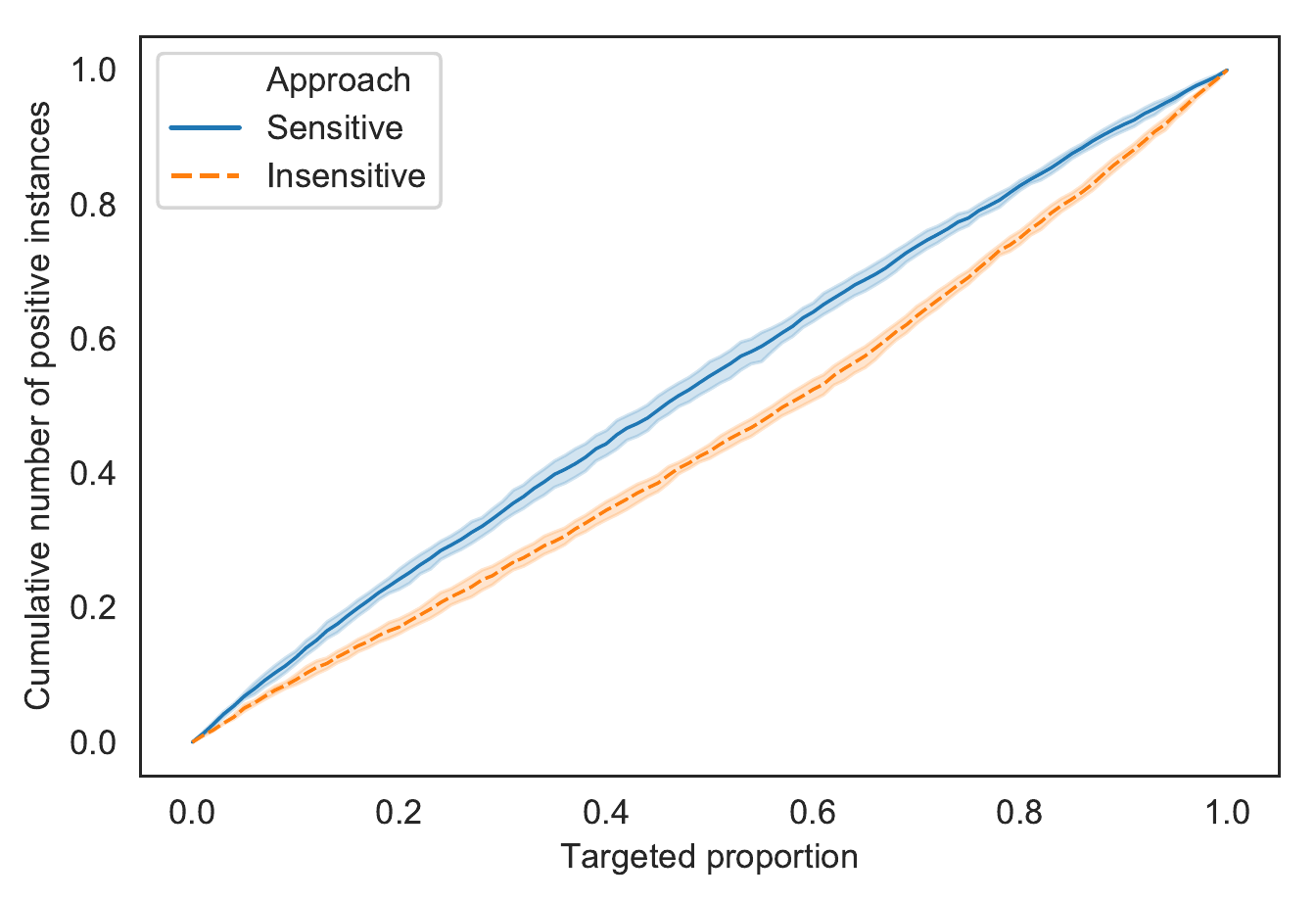}}
 \qquad
 \subfigure[Bank: control.]{\includegraphics[width=0.45\textwidth]{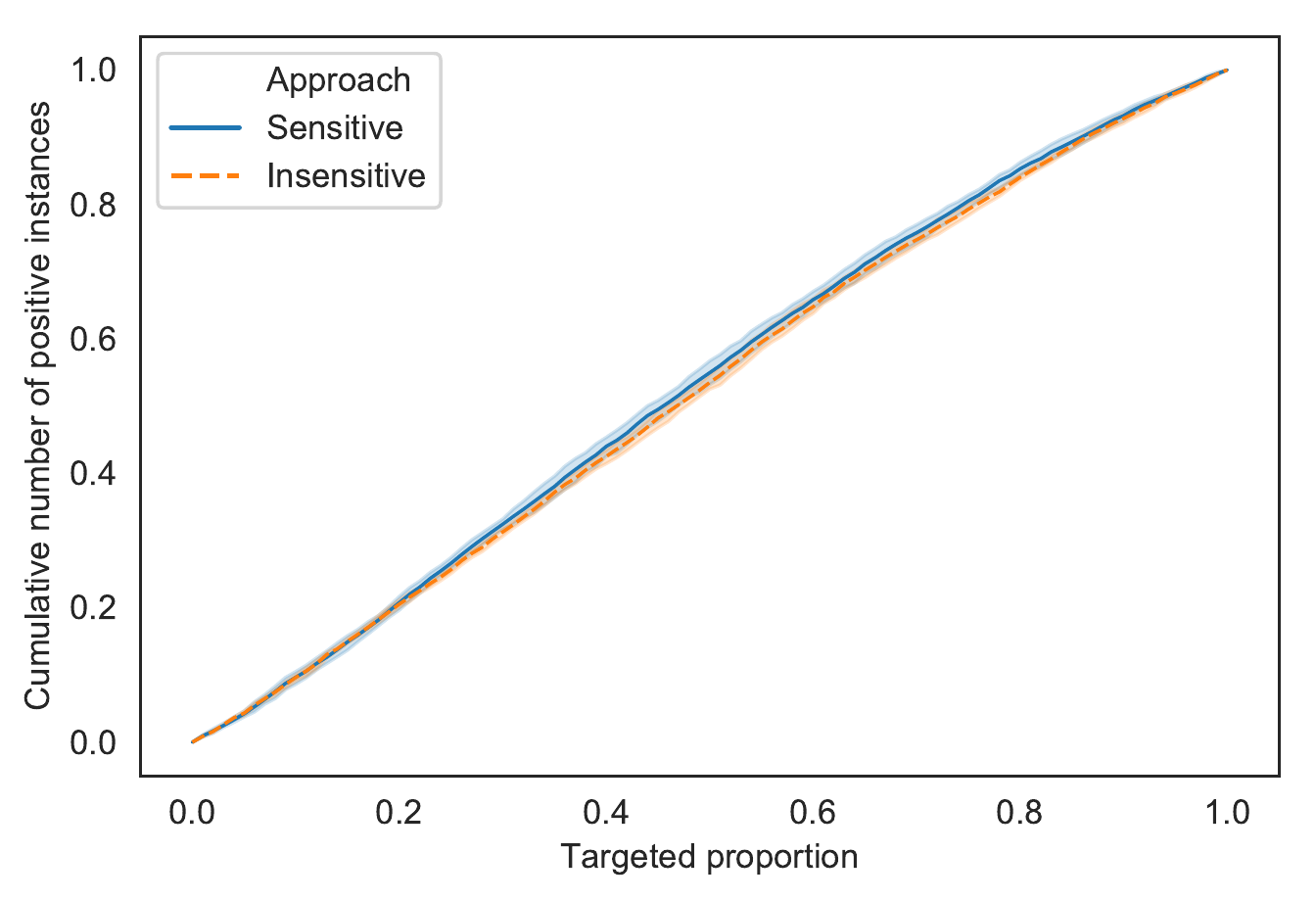}}
 \qquad
 \subfigure[Bank: treatment.]{\includegraphics[width=0.45\textwidth]{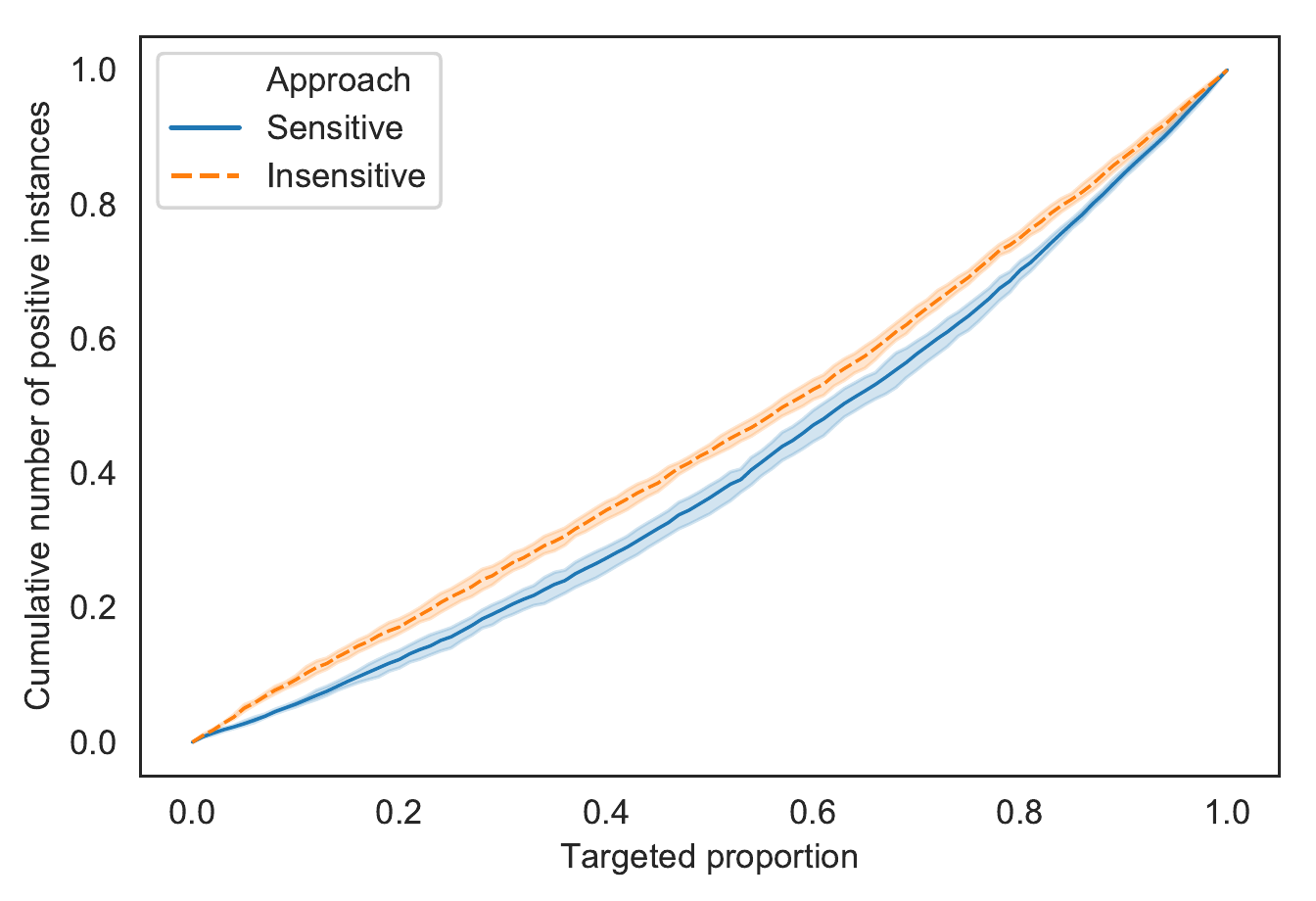}}
 \qquad
 \subfigure[Bank: control.]{\includegraphics[width=0.45\textwidth]{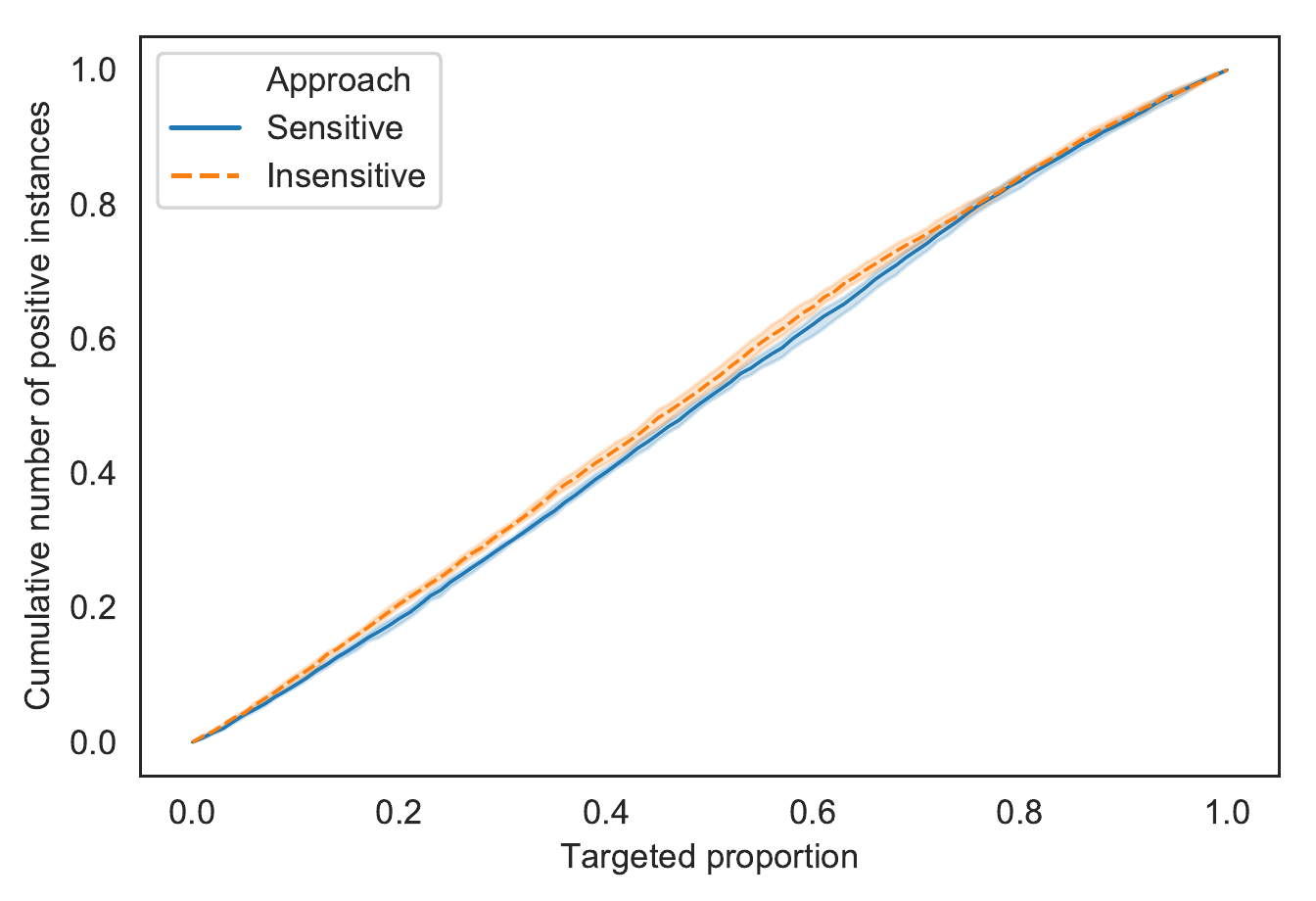}}
 \qquad
 \caption{Cumulative number of positive instances with a positive outcome as a function of the threshold $\tau$ for the Bank data set, for the scenario $b_{11} > b_{10}$ (top) and $b_{11} < b_{10}$ (bottom).}
 \label{fig:cumposplotsa}
\end{figure}

For the scenario $b_{11} > b_{10}$, we observe that the increase in positives when comparing the ECP and ITE ranker is larger for the treatment data set than for the control data set, across all thresholds $\tau$. As a result, $\dot{e}_{11}$ increases more than $\dot{e}_{10}$ in Equation \ref{eq:causalprofitphi}. This increases the positive effect of the ECP ranker on both the average and maximum causal profit that is obtained when $\dot{e}_{11}= \dot{e}_{10}$. 
Since the Qini coefficient is the area under the Qini curve, which itself is the difference between the curves in Figure \ref{fig:cumposplotsa} for treatment and control groups, the difference in the effect of the ECP ranker on the curve of the control and treatment group also causes the Qini coefficient to improve for the scenario $b_{11} > b_{10}$, as observed in Table \ref{tab:qinis}.

For the scenario $b_{11} < b_{10}$, on the other hand, the effect is inverse. The smaller decrease in $\dot{e}_{10}$ causes the average and maximum profit for the ECP ranker to be lower than the ITE ranker and causes the Qini coefficient to decrease. Nonetheless, it is important to acknowledge that the ECP ranker does achieve, at a granular level, the effect that it aims to accomplish, i.e., an increase (in case of $b_{11} > b_{10}$) or decrease (in case of $b_{11} < b_{10}$ ) in the number of positives in the positive treatment class, as shown in Figure \ref{fig:cumposplotsa}.

As an explanation for the effect of the ECP ranker that differs between the treatment and control groups, which should not occur from a theoretical perspective, we conjecture that it results from selection bias. Three arguments support this hypothesis. 
First, in Figure \ref{fig:p11dis}, the distribution of the positive outcome for positive treatment, $P_{11}$, is plotted for the treatment and control subsets of the Bank data set, both for the LR and XGB model. If both data sets were random subsets, then the predictions of the causal classification model should not depend on the group and these distributions should only marginally differ. The observed difference, however, is substantial, which indicates that both data sets have not been randomly sampled. As a result, the obtained estimates for both $P_{11}$ and $t$ are not well calibrated, which is an essential precondition for the ECP ranker to correctly work, since the calculation of expected causal profit values depends on both $P_{11}$ and $t$. 

\begin{figure}[ht]
 \centering
 \subfigure[Bank: LR]{\includegraphics[width=0.45\textwidth]{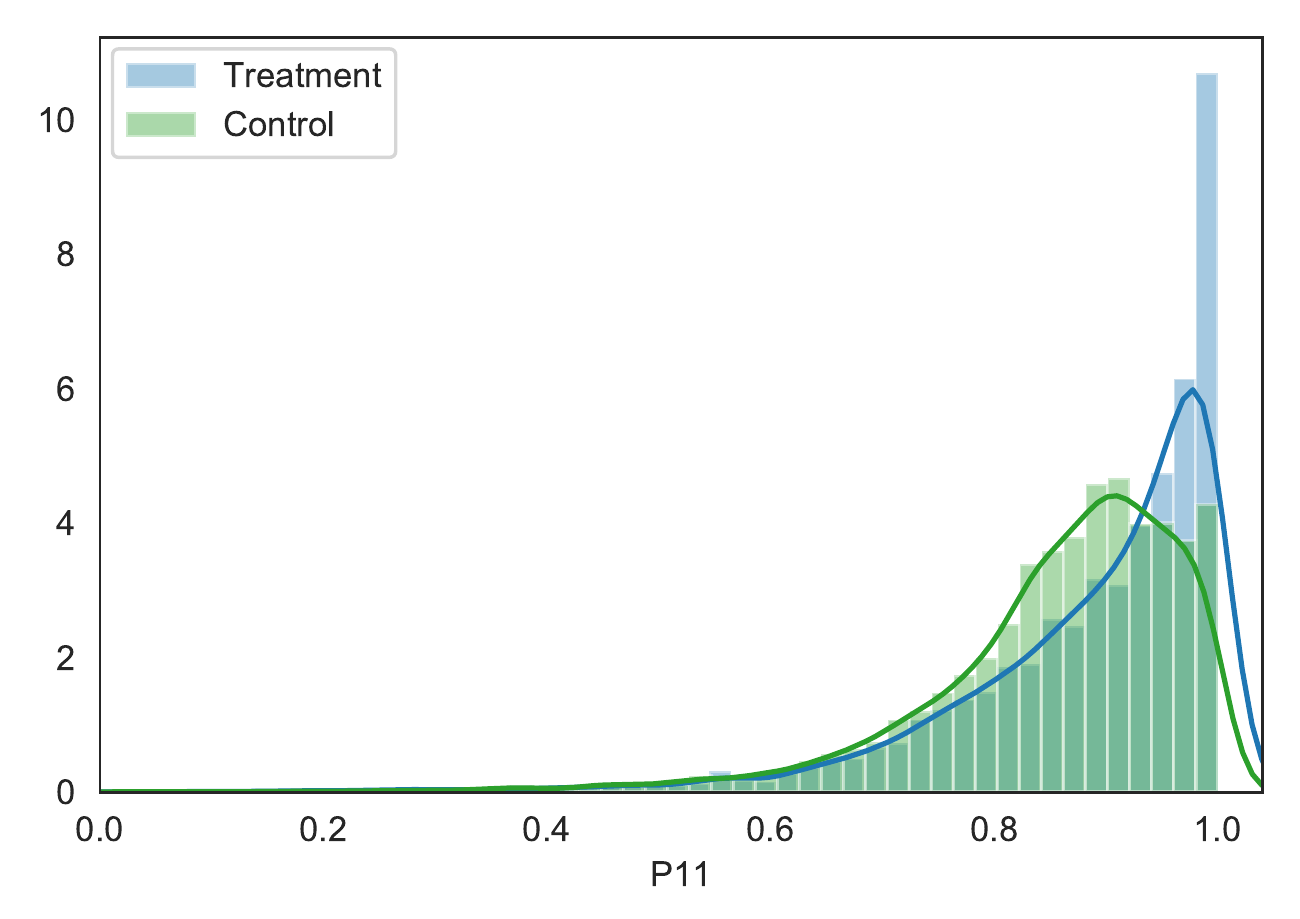}}
 \qquad
 \subfigure[Bank: XGB]{\includegraphics[width=0.45\textwidth]{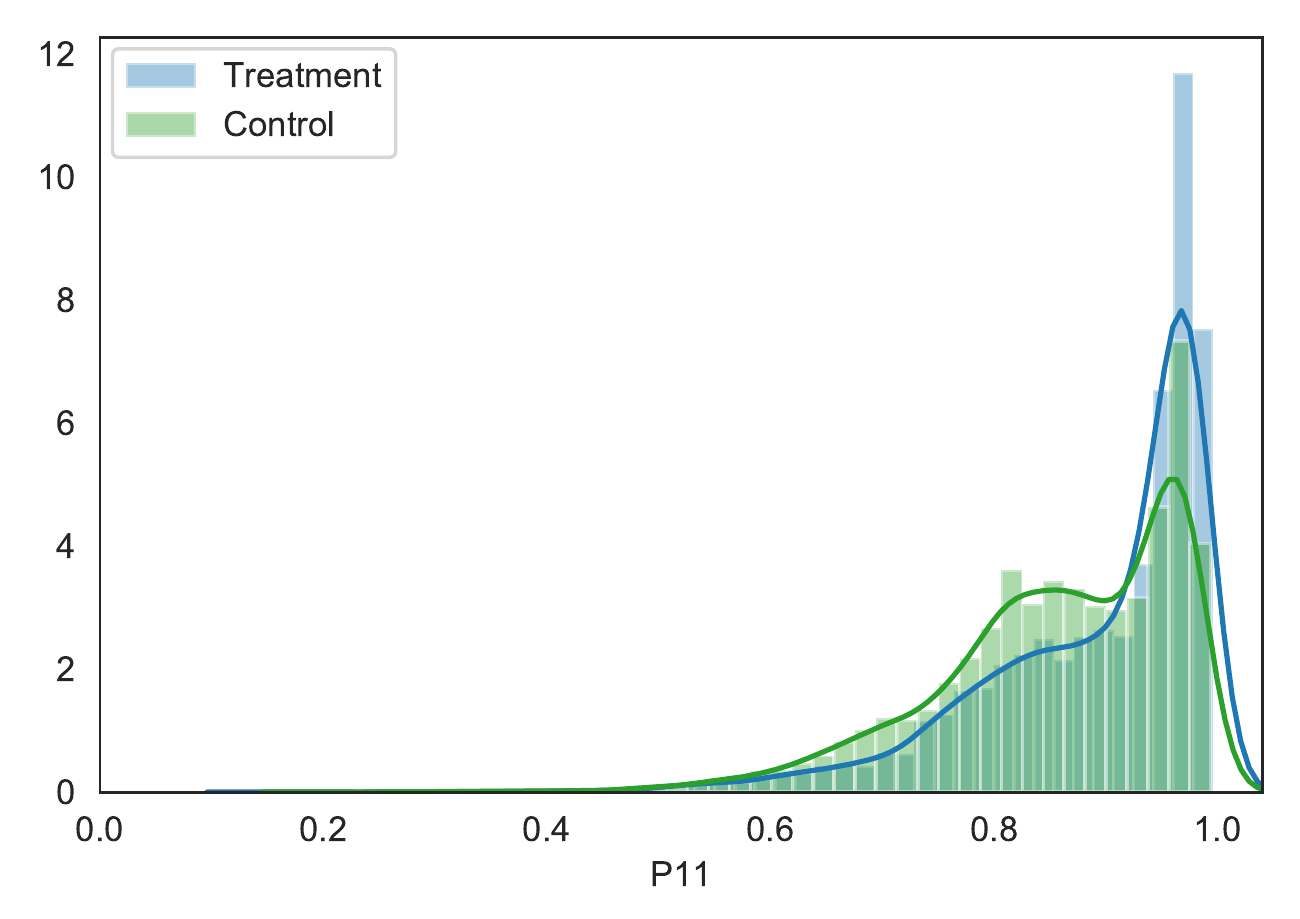}}
 \qquad
 \caption{Distribution of the positive outcome probabilities for positive treatment, $P_{11}$, for the treatment and control data sets}
\label{fig:p11dis}
\end{figure}

Second, the distributions in Figure \ref{fig:p11dis} point toward a preference in the selection of the treatment group for instances with a larger positive outcome probability. This naturally aligns with the objective of marketing campaigns, i.e., boosting the positive outcome rate. Potentially, a predictive model supported the selection process. 

Third, the effect of the ECP ranker on the cumulative proportion of positives for the synthetic data set is effectively approximately the same for treatment and control groups, as shown in Figure \ref{fig:cumposplotsb} and as theoretically expected. Hence, the results on the synthetic data set, which by design is the result of a random trial, support the above hypothesis and indicate that the ECP ranker correctly functions.

\begin{figure}[ht]
 \centering
 \subfigure[Synthetic: treatment.]{\includegraphics[width=0.45\textwidth]{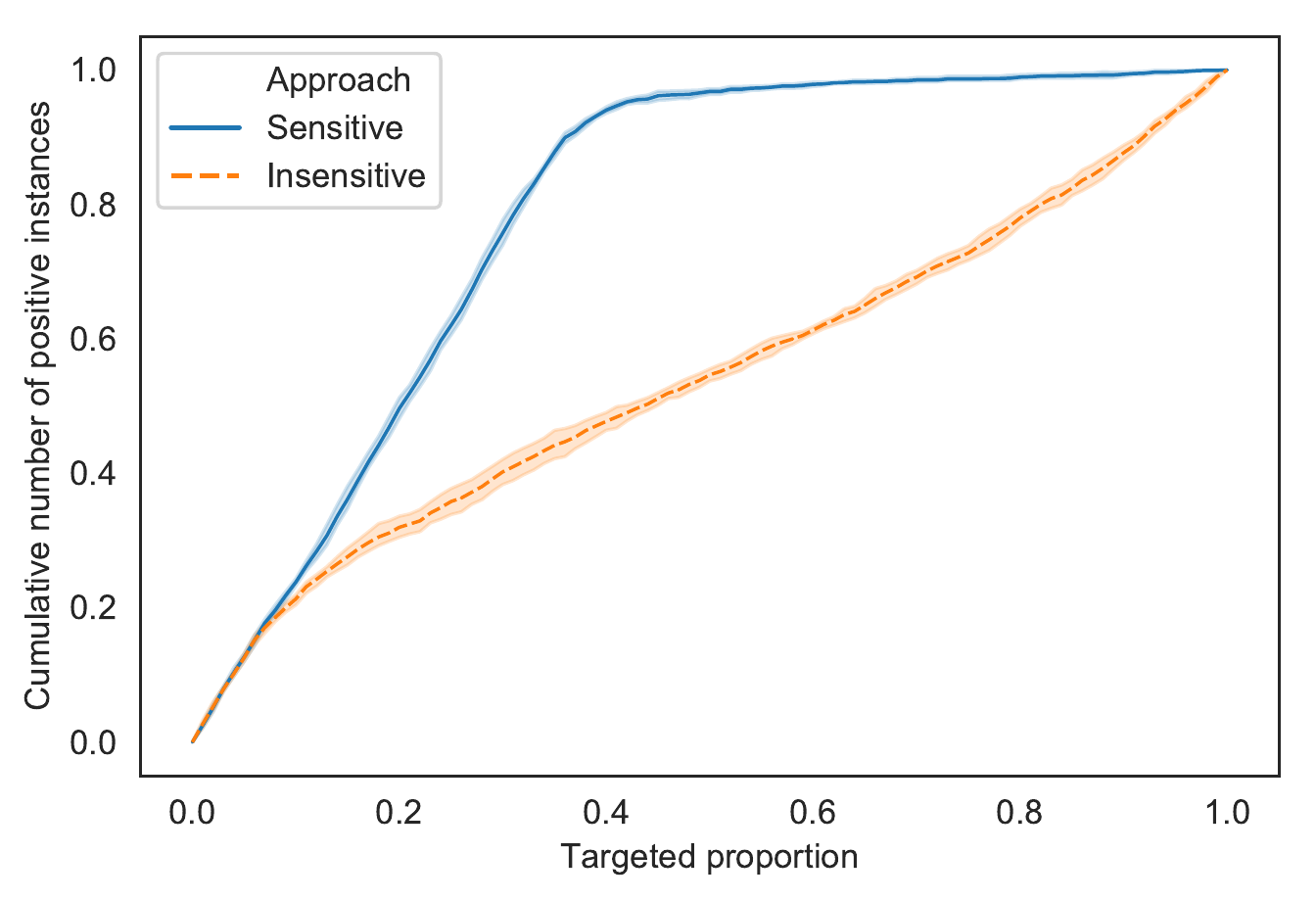}}
 \qquad
 \subfigure[Synthetic: control.]{\includegraphics[width=0.45\textwidth]{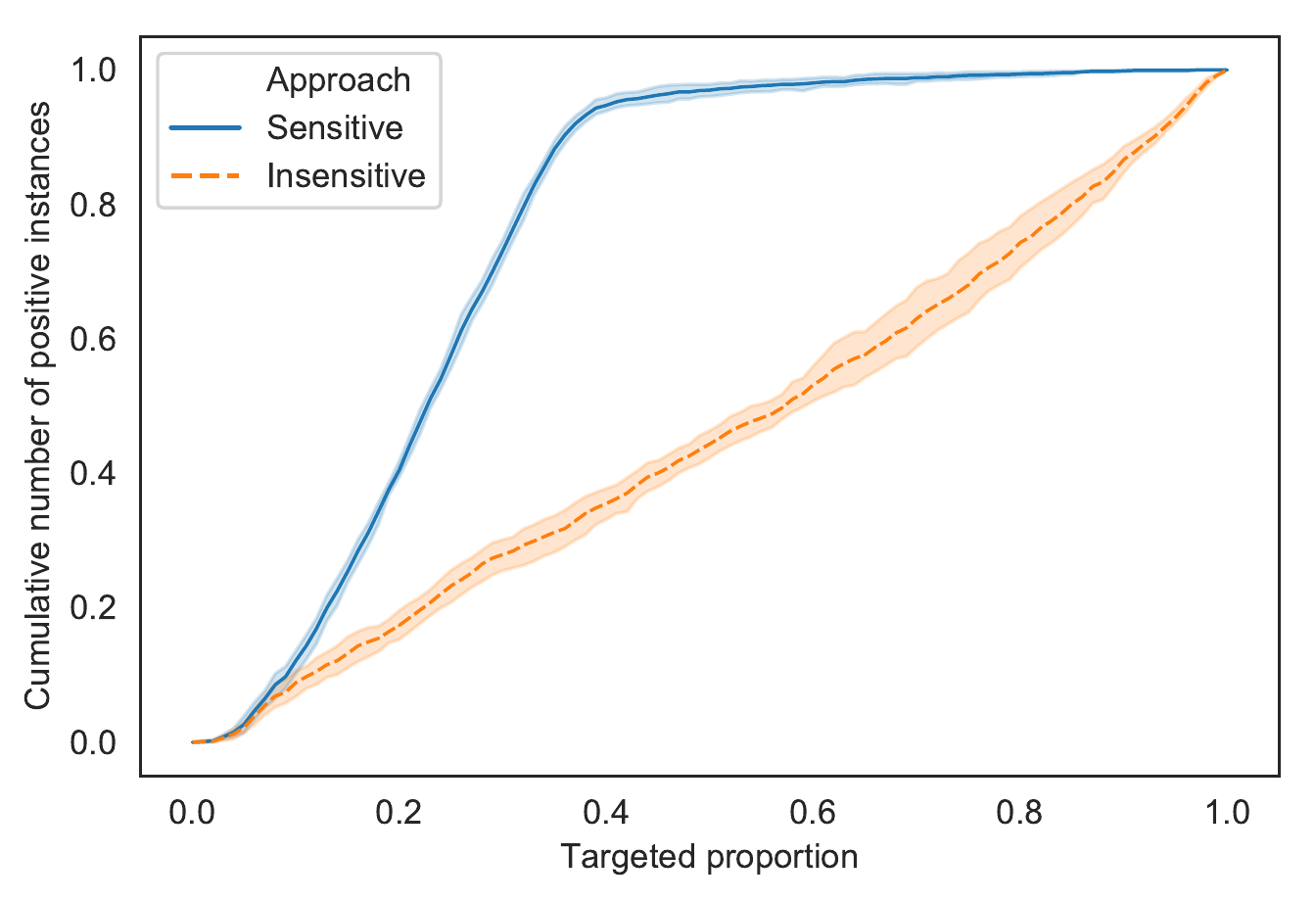}}
 \qquad
 \subfigure[Synthetic: treatment.]{\includegraphics[width=0.45\textwidth]{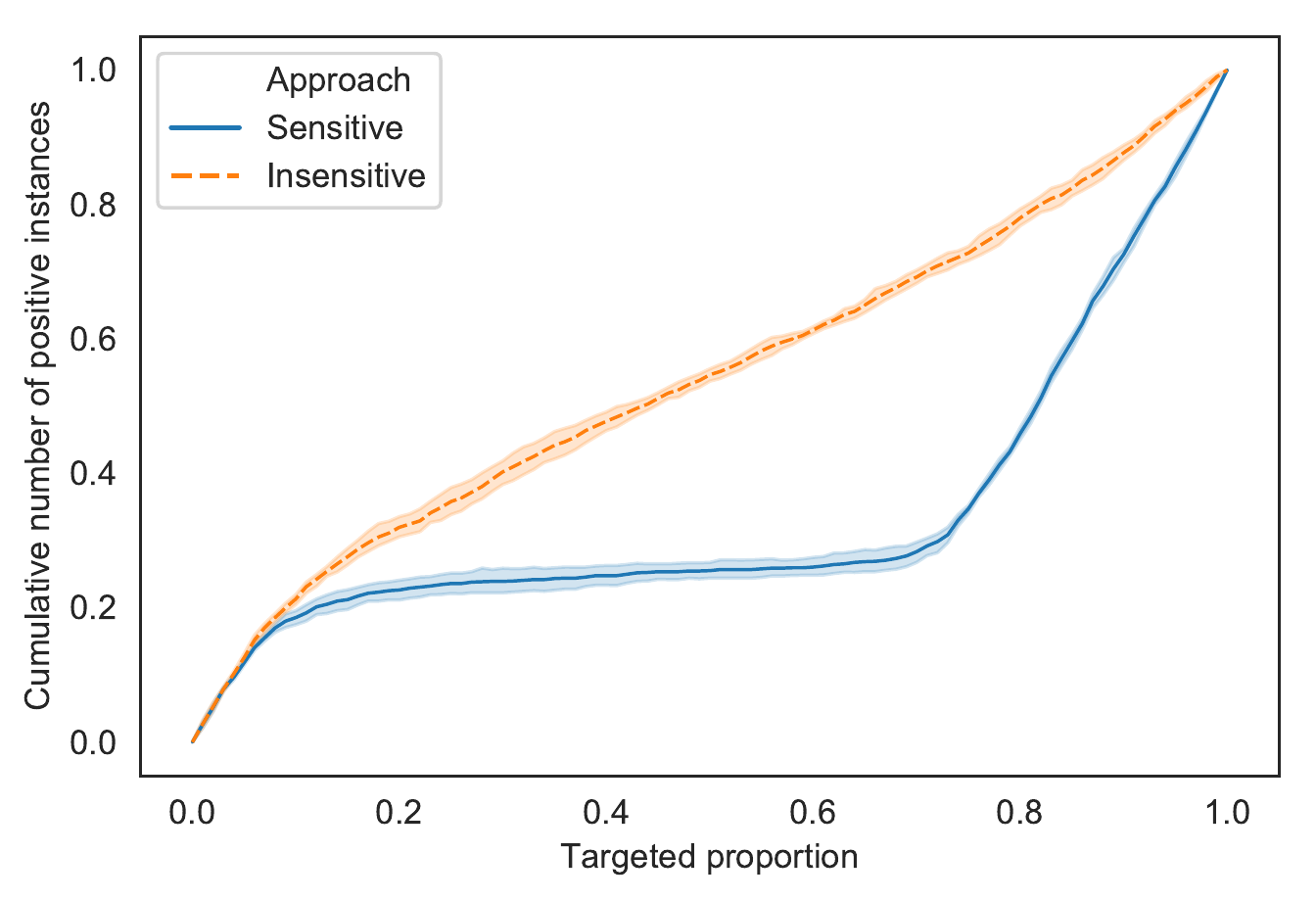}}
 \qquad
 \subfigure[Synthetic: control.]{\includegraphics[width=0.45\textwidth]{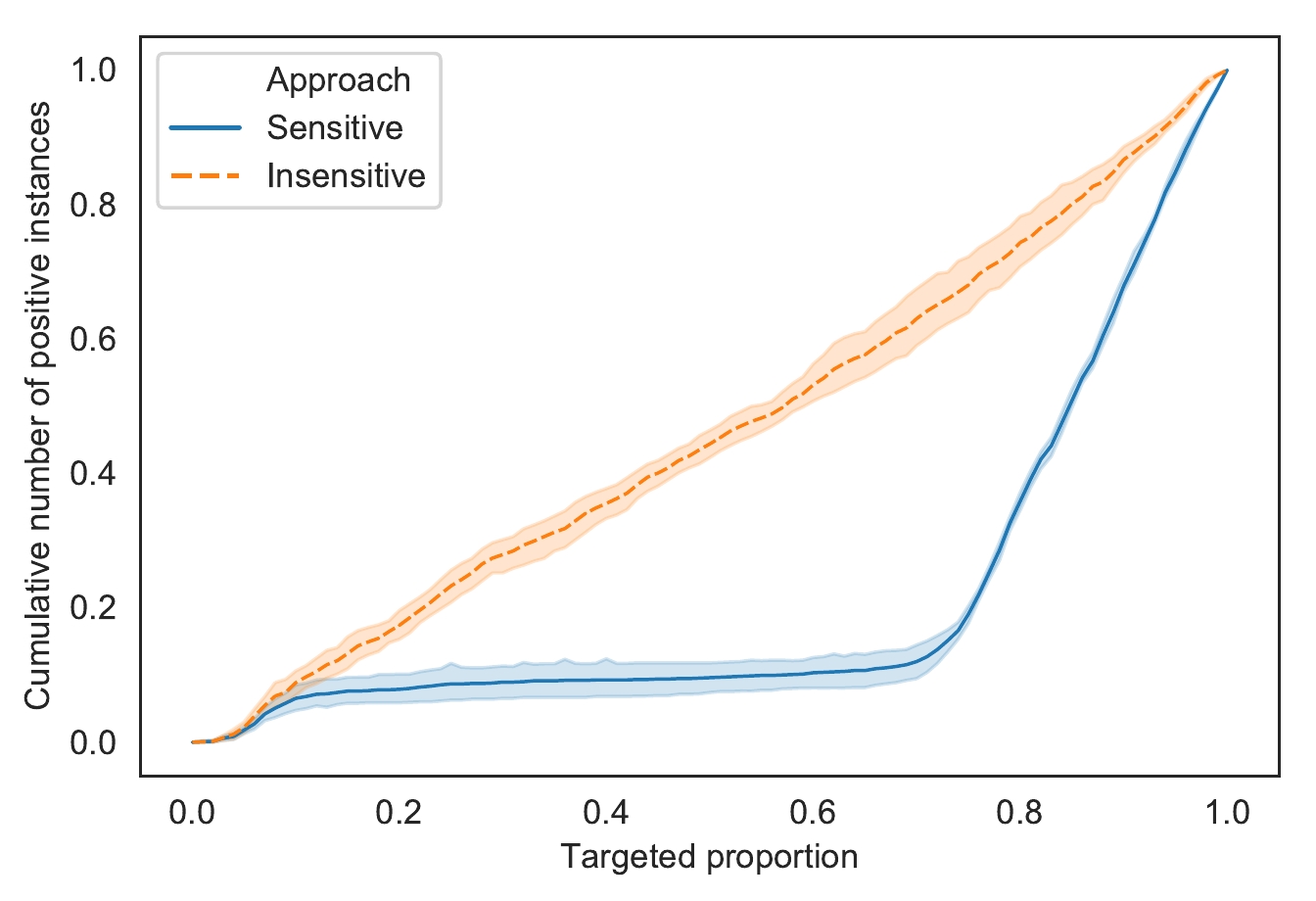}}
 \qquad
 \caption{Cumulative number of positive instances with a positive outcome as a function of the threshold $\tau$ for the Synthetic data set, for the scenario $b_{11} > b_{10}$ (top) and $b_{11} < b_{10}$ (bottom).}
\label{fig:cumposplotsb}
\end{figure}

To accommodate the requirement of having well-calibrated estimates and to use the ECP ranker (which encompasses the cost-sensitive causal classification decision boundary) with observational data, an important direction for future research is to explore the use of causal classification methods that allow addressing potential selection bias (as mentioned in Footnote 2) or, alternatively, calibration methods for postprocessing the estimates of $P_{11}$ and $t$. 

\section{Conclusions and future research}
\label{sec:conclusions}

Double binary causal classification models can be adopted across a variety of operational business processes to predict the effect of a binary treatment on a binary business outcome, as a function of the characteristics of the process instance. This allows optimizing operational decision-making and selecting the optimal treatment to apply for each individual instance to maximize the positive outcome rate. However, while in the literature, a variety of powerful approaches have been presented for learning individual treatment effects, no formal framework for optimal decision-making has been elaborated in presence of cost of treatment and benefit of outcome information.  

In this article, we introduce a formal decision-making framework that takes into account the estimated ITE, the treatment-conditional positive outcome probabilities, and the costs and benefits of the treatments and outcomes. We build on the expected value framework (e.g., see Chapter 11 in \cite{provost2013data}) and extend the established cost-sensitive classification approach introduced by \cite{elkan2001foundations} and derive the cost-sensitive causal classification decision boundary for causally classifying instances in the positive or negative treatment class, to maximize the expected causal profit. Whereas the cost-insensitive causal classification decision boundary for maximizing the positive outcome rate merely assesses the estimated ITE, the cost-sensitive decision boundary is a linear function of the ITE and the positive outcome probability given the application of the positive treatment and is characterized by the problem-specific cost and benefit parameters. 

The presented cost-sensitive decision-making framework implies a novel approach for ranking and selecting instances that are to be treated with the positive treatment. Traditionally, instances are ranked based on the ITE, whereas the presented expected causal profit ranker takes into account the ITE, the positive outcome probability given the application of the positive treatment, as well as the cost and benefit parameters. The ECP ranker ranks instances with the objective of maximizing the expected causal profit for each possible number of selected instances and supports the selection of the optimal subset in the positive treatment class. 

We empirically evaluate the proposed ranking approach on synthetic and marketing data. Logistic regression and XGBoost are applied to estimate individual treatment effects in combination with three causal modeling metalearning methods. We evaluate model performance in terms of the Qini coefficient as well as average and maximum causal profit. For ranking instances, we adopt the traditional, cost-insensitive approach and the newly introduced ECP ranker. The experimental results acknowledge the potential of the ECP ranker to improve the profitability, although well-calibrated estimates of the positive outcome probability for the positive treatment and the ITE are preconditions for the presented approach to work correctly. Note that this condition also applies to the use of causal classification methods and the ITE ranker when potential selection bias is not addressed in learning a model. 
As a digital appendix to this article, we provide the experimental code that allows replicating the reported results. Since we strongly believe the presented framework opens vast opportunities for further research and practical application, we therefore hope to promote further research and facilitate the adoption of the framework by fellow scientists and practitioners.

Since the results are limited to double binary causal classification, a first stream of further research is to formulate the cost-sensitive decision boundary and ranking approach for multitreatment and multiclass outcome causal classification problems, as well as for continuous treatments (e.g., discount amount or price, time, dosage, temperature, etc.) and continuous outcome (e.g., amount or time spent, yield, number of visits, etc.) causal learning problems \cite{gubela2020response}. 
A second stream of further research is to explore alternative and more complex cost and benefit parameter setups, which in a realistic setting may be dependent on the proportion of targeted instances, stochastic in nature, or may be instance-dependent instead of class-dependent. 
A third stream concerns causal model learning. While the proposed framework allows  factoring in costs and benefits in a postprocessing stage similar to cost-sensitive learning methods, they may be taken into account during causal model development. Additionally, the effect of cost-sensitive causal learning on the performance and stability of causal models in cases of imbalanced data sets, both in terms of the class and the treatment distribution, remains to be explored. 
Moreover, the use of advanced causal machine learning methods that are able to efficiently and effectively handle selection bias is to be evaluated for obtaining calibrated estimates of outcome probabilities and individual treatment effects. 
Finally, further empirical research is necessary to validate and explore the use of the proposed evaluation and ranking approach across various decision-making settings.

\textbf{Acknowledgment}: The authors acknowledge the support of Innoviris, the Brussels Region Research funding agency.

\bibliographystyle{plain}
\bibliography{biblio}

\begin{thebibliography}{10}

\bibitem{ascarza2018retention}
Eva Ascarza.
\newblock Retention futility: Targeting high-risk customers might be
  ineffective.
\newblock {\em Journal of Marketing Research}, 55(1):80--98, 2018.

\bibitem{athey2019generalized}
Susan Athey, Julie Tibshirani, Stefan Wager, et~al.
\newblock Generalized random forests.
\newblock {\em The Annals of Statistics}, 47(2):1148--1178, 2019.

\bibitem{ben2018optimal}
Tarek Ben~Rhouma and Georges Zaccour.
\newblock Optimal marketing strategies for the acquisition and retention of
  service subscribers.
\newblock {\em Management Science}, 64(6):2609--2627, 2018.

\bibitem{berrevoets2020organite}
Jeroen Berrevoets, James Jordon, Ioana Bica, Mihaela van~der Schaar, et~al.
\newblock Organite: Optimal transplant donor organ offering using an individual
  treatment effect.
\newblock {\em Advances in Neural Information Processing Systems}, 33, 2020.

\bibitem{berrevoets2019optimising}
Jeroen Berrevoets, Sam Verboven, and Wouter Verbeke.
\newblock Optimising individual-treatment-effect using bandits.
\newblock {\em arXiv preprint arXiv:1910.07265}, 2019.

\bibitem{caliendo2008some}
Marco Caliendo and Sabine Kopeinig.
\newblock Some practical guidance for the implementation of propensity score
  matching.
\newblock {\em Journal of economic surveys}, 22(1):31--72, 2008.

\bibitem{chen2020causalml}
Huigang Chen, Totte Harinen, Jeong-Yoon Lee, Mike Yung, and Zhenyu Zhao.
\newblock Causalml: Python package for causal machine learning.
\newblock {\em arXiv preprint arXiv:2002.11631}, 2020.

\bibitem{chen2016xgboost}
Tianqi Chen and Carlos Guestrin.
\newblock Xgboost: A scalable tree boosting system.
\newblock In {\em Proceedings of the 22nd ACM SIGKDD International Conference
  on Knowledge Discovery and Data Mining}, pages 785--794, 2016.

\bibitem{davis2017using}
Jonathan Davis and Sara~B Heller.
\newblock Using causal forests to predict treatment heterogeneity: An
  application to summer jobs.
\newblock {\em American Economic Review}, 107(5):546--50, 2017.

\bibitem{debaere2019reducing}
Steven Debaere, Floris Devriendt, Johanna Brunneder, Wouter Verbeke, Tom
  De~Ruyck, and Kristof Coussement.
\newblock Reducing inferior member community participation using uplift
  modeling: Evidence from a field experiment.
\newblock {\em Decision Support Systems}, 123:113077, 2019.

\bibitem{devriendt2019you}
Floris Devriendt, Jeroen Berrevoets, and Wouter Verbeke.
\newblock Why you should stop predicting customer churn and start using uplift
  models.
\newblock {\em Information Sciences}, 2019.

\bibitem{devriendt2020learning}
Floris Devriendt, Tias Guns, and Wouter Verbeke.
\newblock Learning to rank for uplift modeling.
\newblock {\em IEEE Transactions on Knowledge and Data Engineering}, 2020.

\bibitem{diemert2018large}
Eustache Diemert, Artem Betlei, Christophe Renaudin, and Massih-Reza Amini.
\newblock A large scale benchmark for uplift modeling.
\newblock In {\em KDD}, 2018.

\bibitem{elkan2001foundations}
Charles Elkan.
\newblock The foundations of cost-sensitive learning.
\newblock In {\em International Joint Conference on Artificial Intelligence},
  volume~17, pages 973--978. Lawrence Erlbaum Associates Ltd, 2001.

\bibitem{fernandez2019causal}
Carlos Fern{\'a}ndez and Foster Provost.
\newblock Causal classification: Treatment effect vs. outcome prediction.
\newblock {\em Outcome Prediction (June 22, 2019)}, 2019.

\bibitem{ganin2016domain}
Yaroslav Ganin, Evgeniya Ustinova, Hana Ajakan, Pascal Germain, Hugo
  Larochelle, Fran{\c{c}}ois Laviolette, Mario Marchand, and Victor Lempitsky.
\newblock Domain-adversarial training of neural networks.
\newblock {\em The Journal of Machine Learning Research}, 17(1):2096--2030,
  2016.

\bibitem{gubela2020response}
Robin~M Gubela, Stefan Lessmann, and Szymon Jaroszewicz.
\newblock Response transformation and profit decomposition for revenue uplift
  modeling.
\newblock {\em European Journal of Operational Research}, 283(2):647--661,
  2020.

\bibitem{gupta2020maximizing}
Vishal Gupta, Brian~Rongqing Han, Song-Hee Kim, and Hyung Paek.
\newblock Maximizing intervention effectiveness.
\newblock {\em Management Science}, 2020.

\bibitem{guyon2003design}
Isabelle Guyon.
\newblock Design of experiments of the nips 2003 variable selection benchmark.
\newblock In {\em NIPS 2003 workshop on feature extraction and feature
  selection}, volume 253, 2003.

\bibitem{hill2011bayesian}
Jennifer~L Hill.
\newblock Bayesian nonparametric modeling for causal inference.
\newblock {\em Journal of Computational and Graphical Statistics},
  20(1):217--240, 2011.

\bibitem{hillstrom2008minethatdata}
Kevin Hillstrom.
\newblock The minethatdata e-mail analytics and data mining challenge.
\newblock {\em MineThatData blog}, 2008.

\bibitem{holland1986statistics}
Paul~W Holland.
\newblock Statistics and causal inference.
\newblock {\em Journal of the American Statistical Association},
  81(396):945--960, 1986.

\bibitem{kunzel2019metalearners}
S{\"o}ren~R K{\"u}nzel, Jasjeet~S Sekhon, Peter~J Bickel, and Bin Yu.
\newblock Metalearners for estimating heterogeneous treatment effects using
  machine learning.
\newblock {\em Proceedings of the National Academy of Sciences},
  116(10):4156--4165, 2019.

\bibitem{kuusisto2014support}
Finn Kuusisto, Vitor~Santos Costa, Houssam Nassif, Elizabeth Burnside, David
  Page, and Jude Shavlik.
\newblock Support vector machines for differential prediction.
\newblock In {\em Joint European Conference on Machine Learning and Knowledge
  Discovery in Databases}, pages 50--65. Springer, 2014.

\bibitem{le2020xgboost}
Nguyen Quoc~Khanh Le, Duyen~Thi Do, Fang-Ying Chiu, Edward Kien~Yee Yapp,
  Hui-Yuan Yeh, and Cheng-Yu Chen.
\newblock Xgboost improves classification of mgmt promoter methylation status
  in idh1 wildtype glioblastoma.
\newblock {\em Journal of Personalized Medicine}, 10(3):128, 2020.

\bibitem{lo2002true}
Victor~SY Lo.
\newblock The true lift model: a novel data mining approach to response
  modeling in database marketing.
\newblock {\em ACM SIGKDD Explorations Newsletter}, 4(2):78--86, 2002.

\bibitem{luo2019and}
Xueming Luo, Xianghua Lu, and Jing Li.
\newblock When and how to leverage e-commerce cart targeting: The relative and
  moderated effects of scarcity and price incentives with a two-stage field
  experiment and causal forest optimization.
\newblock {\em Information Systems Research}, 30(4):1203--1227, 2019.

\bibitem{econml}
{Microsoft Research}.
\newblock {EconML}: {A Python Package for ML-Based Heterogeneous Treatment
  Effects Estimation}.
\newblock https://github.com/microsoft/EconML, 2019.
\newblock Version 0.x.

\bibitem{olaya2020uplift}
Diego Olaya, Jonathan V{\'a}squez, Sebasti{\'a}n Maldonado, Jaime Miranda, and
  Wouter Verbeke.
\newblock Uplift modeling for preventing student dropout in higher education.
\newblock {\em Decision Support Systems}, page 113320, 2020.

\bibitem{pearl2009causality}
Judea Pearl.
\newblock {\em Causality}.
\newblock Cambridge university press, 2009.

\bibitem{powers2018some}
Scott Powers, Junyang Qian, Kenneth Jung, Alejandro Schuler, Nigam~H Shah,
  Trevor Hastie, and Robert Tibshirani.
\newblock Some methods for heterogeneous treatment effect estimation in high
  dimensions.
\newblock {\em Statistics in Medicine}, 37(11):1767--1787, 2018.

\bibitem{provost2013data}
Foster Provost and Tom Fawcett.
\newblock {\em Data Science for Business: What you need to know about data
  mining and data-analytic thinking}.
\newblock " O'Reilly Media, Inc.", 2013.

\bibitem{radcliffe2007using}
Nicholas~J Radcliffe.
\newblock Using control groups to target on predicted lift: Building and
  assessing uplift models.
\newblock {\em Direct Marketing Analytics Journal}, 1(3):14--21, 2007.

\bibitem{rombaut2020effectiveness}
Evy Rombaut and Marie-Anne Guerry.
\newblock The effectiveness of employee retention through an uplift modeling
  approach.
\newblock {\em International Journal of Manpower}, 2020.

\bibitem{rosenbaum1983central}
Paul~R Rosenbaum and Donald~B Rubin.
\newblock The central role of the propensity score in observational studies for
  causal effects.
\newblock {\em Biometrika}, 70(1):41--55, 1983.

\bibitem{rubin1974estimating}
Donald~B Rubin.
\newblock Estimating causal effects of treatments in randomized and
  nonrandomized studies.
\newblock {\em Journal of Educational Psychology}, 66(5):688, 1974.

\bibitem{rubin1978bayesian}
Donald~B Rubin.
\newblock Bayesian inference for causal effects: The role of randomization.
\newblock {\em The Annals of statistics}, pages 34--58, 1978.

\bibitem{rudas2018linear}
Krzysztof Ruda{\'s} and Szymon Jaroszewicz.
\newblock Linear regression for uplift modeling.
\newblock {\em Data Mining and Knowledge Discovery}, 32(5):1275--1305, 2018.

\bibitem{shalit2017estimating}
Uri Shalit, Fredrik~D Johansson, and David Sontag.
\newblock Estimating individual treatment effect: generalization bounds and
  algorithms.
\newblock In {\em International Conference on Machine Learning}, pages
  3076--3085. PMLR, 2017.

\bibitem{sheng2006thresholding}
Victor~S Sheng and Charles~X Ling.
\newblock Thresholding for making classifiers cost-sensitive.
\newblock In {\em AAAI}, pages 476--481, 2006.

\bibitem{shmueli2010explain}
Galit Shmueli et~al.
\newblock To explain or to predict?
\newblock {\em Statistical science}, 25(3):289--310, 2010.

\bibitem{simester2020targeting}
Duncan Simester, Artem Timoshenko, and Spyros~I Zoumpoulis.
\newblock Targeting prospective customers: Robustness of machine-learning
  methods to typical data challenges.
\newblock {\em Management Science}, 66(6):2495--2522, 2020.

\bibitem{splawa1990application}
Jerzy Splawa-Neyman, Dorota~M Dabrowska, and TP~Speed.
\newblock On the application of probability theory to agricultural experiments.
  essay on principles. section 9.
\newblock {\em Statistical Science}, pages 465--472, 1990.

\bibitem{turkyilmaz2018causal}
Ali Turkyilmaz, Leyla Temizer, and Asil Oztekin.
\newblock A causal analytic approach to student satisfaction index modeling.
\newblock {\em Annals of Operations Research}, 263(1-2):565--585, 2018.

\bibitem{verbeke2020foundations}
Wouter Verbeke, Diego Olaya, Jeroen Berrevoets, and Sebasti{\'a}n Maldonado.
\newblock The foundations of cost-sensitive causal classification.
\newblock {\em arXiv preprint arXiv:2007.12582}, 2020.

\bibitem{wager2018estimation}
Stefan Wager and Susan Athey.
\newblock Estimation and inference of heterogeneous treatment effects using
  random forests.
\newblock {\em Journal of the American Statistical Association},
  113(523):1228--1242, 2018.

\bibitem{zadrozny2001obtaining}
Bianca Zadrozny and Charles Elkan.
\newblock Obtaining calibrated probability estimates from decision trees and
  naive bayesian classifiers.
\newblock In {\em Icml}, volume~1, pages 609--616. Citeseer, 2001.

\bibitem{zhao2012estimating}
Yingqi Zhao, Donglin Zeng, A~John Rush, and Michael~R Kosorok.
\newblock Estimating individualized treatment rules using outcome weighted
  learning.
\newblock {\em Journal of the American Statistical Association},
  107(499):1106--1118, 2012.

\bibitem{zhao2019uplift}
Zhenyu Zhao and Totte Harinen.
\newblock Uplift modeling for multiple treatments with cost optimization.
\newblock In {\em 2019 IEEE International Conference on Data Science and
  Advanced Analytics (DSAA)}, pages 422--431. IEEE, 2019.

\end{thebibliography}
\end{document}